\documentclass[11pt]{article}
\usepackage{fullpage}
\usepackage[colorlinks,citecolor=red,urlcolor=blue,linkcolor=blue]{hyperref}
\usepackage{amssymb,amsbsy,amsfonts,amsmath,amsthm,mathtools}
\RequirePackage[numbers]{natbib}
\usepackage{mathrsfs}
\usepackage{graphicx}
\usepackage{epstopdf}
\usepackage{epsfig}
\usepackage{framed}
\usepackage{url}
\usepackage{color}
\usepackage{bm, times, natbib, graphics}
\usepackage[ruled,vlined]{algorithm2e}

\newtheorem{lemma}{Lemma}[section]
\newtheorem{proposition}{Proposition}[section]
\newtheorem{thm}{Theorem}[section]

\everymath{\displaystyle}

{\unskip\nobreak\hskip 1em plus 1fil\nobreak$\Box$
\parfillskip=0pt%
\endtrivlist}

\def\text#1{\mbox{\rm #1}}

\DeclarePairedDelimiter{\ceil}{\lceil}{\rceil}

\newcommand{\Norm}[1]{\left\|{#1} \right\|}

\let\oldsqrt\sqrt
\def\sqrt{\mathpalette\DHLhksqrt}
\def\DHLhksqrt#1#2{%
\setbox0=\hbox{$#1\oldsqrt{#2\,}$}\dimen0=\ht0
\advance\dimen0-0.2\ht0
\setbox2=\hbox{\vrule height\ht0 depth -\dimen0}%
{\box0\lower0.4pt\box2}}

\begin{document}

\setkeys{Gin}{width=1.0\textwidth}



\title{Minimax Optimal Convergence Rates for Estimating Ground Truth from Crowdsourced Labels}


\author{
	Chao Gao%
	\thanks{
	Department of Statistics, Yale University, New Haven, CT 06520. The work was done when this author was an intern at Microsoft Research. Email: {\tt chao.gao@yale.edu}. }
	\and
	Dengyong Zhou
	\thanks{Microsoft Research, Redmond, WA 98052. Email: {\tt dengyong.zhou@microsoft.com}.}
}

\maketitle

\begin{center}
\textbf{Abstract}
\end{center}
Crowdsourcing has become a primary means for label collection in  many real-world machine learning applications. A classical method for inferring the true labels from the noisy labels provided by crowdsourcing workers is Dawid-Skene estimator. In this paper, we prove convergence rates of a projected EM algorithm for the Dawid-Skene estimator. The revealed exponent in the rate of convergence is shown to be optimal via a lower bound argument. Our work resolves the long standing issue of whether Dawid-Skene estimator has sound theoretical guarantees besides its good performance observed in practice. In addition, a comparative study with majority voting illustrates both advantages and pitfalls of the Dawid-Skene estimator.

\vspace*{.3in}


%


\section{Introduction}

In many real-world machine learning applications, from protein structure prediction to web-scale image categorization, crowdsourcing has become a primary way to obtain large amounts of labeled data \citep{vondab04,howe2006rise,doan2011crowdsourcing,Boh14}. There are many commercial web services for crowdsourcing. Among them,  Amazon Mechanical Turk\footnote{https://www.mturk.com/mturk/welcome} is perhaps the most popular one.  Crowdsourcing requesters load their labeling tasks into a crowdsourcing site together with their labeling guidelines and payment instructions. In the meantime, millions of crowdsourcing workers worldwide associated with the crowdsourcing platform pick the tasks that they are interested to work on. Usually, a requester may be able to obtain hundreds of thousands of labels in few hours with only one cent per label.

Despite the substantial advantages of crowdsourcing in terms of cost and time savings, the collected labels may be very noisy since crowdsourcing workers are often unskillful and even some of them can be spammers. To overcome the label quality issue, a requester lets each item be redundantly labeled by several different workers. Such a quality control solution immediately gives rise to a fundamental challenge in crowdsourcing: estimating the true labels from noisy but redundant worker labels.   Assume that there are  $m$ items with unknown true labels in $\{0, 1\}.$ Denote by $y^*_j$ the true label of the $j$-th item. These items are independently labeled by $n$ workers. Denote all the worker labels  by an $n\times m$ matrix $X$ in which the $(i,j)$-th element represents the label that the $i$-th worker assigns to the $j$-th item. Then, our task is to accurately estimate all the unknown $y^*_j$  using the observed $X$.

A classical method for estimating the truth labels from the noisy crowdsourced labels is devised  by Dawid and Skene \cite{DawSke79}. In its simplified form, each worker  is assumed to have a number $p_i^*\in [0, 1]$ to  characterize her intrinsic labeling ability: for any given item, with probability $p_i^*, $ the label from the $i$-th worker is correct,  and with probability of $1 - p_i^*$,  the label from the $i$-th worker is wrong.  Workers' abilities can be estimated by maximizing the marginal likelihood, and an estimate of the true labels follows by  plugging the estimate of workers' abilities  into  Bayes's rule. Moreover, by regarding the true labels as latent variables, this two-stage estimation can be iteratively solved through the Expectation-Maximization (EM) algorithm \citep{DemLaiRub77}. Both E-step and M-step have simple closed-form solutions.

While Dawid-Skene estimator and its numerous variants have been widely used in practice \citep{PaFaBu95, SnoCon08, RayYuZha10, LiuPenIhl12,zhoplaby12,  CheLinZho13}, it is surprising that there is no  theoretical analysis on its statistical property. In this paper, we address this issue mainly from two aspects: (1)  analyzing the convergence rates of the EM algorithm; (2) establishing the minimax lower bounds of the problem.

Our theoretical analysis shows that the error rate of the Dawid-Skene estimator is exponentially small and we identify a new exponent which consists of two critical quantities characterizing the collective wisdom of the crowd. We show that the exponent cannot be improved in a minimax sense by providing a lower bound argument for each of the two quantities.



As a byproduct of the theoretical study of label estimator, we show that  Dawid-Skene estimator  also provides an accurate estimator of the workers' abilities. Non-asymptotic bounds are derived for estimation error in both average and maximum losses. We also derive the exact asymptotic distribution of the estimator for any finite subset of workers. Finally, a high-dimensional central limit theorem is derived for the joint distribution of all workers.

The theoretical analysis of Dawid-Skene estimator is followed by a comparative study with the simple majority voting estimator. We construct a concrete example showing that when the majority of the crowd are spammers, the majority voting is inconsistent while Dawid-Skene estimator still converges exponentially fast. On the other hand, the majority voting is robust to model misspecification, unlike Dawid-Skene estimator, which is sensitive to the model and may suffer a certain loss when the model is misspecified.

In the literature of crowdsourcing, there has been little work on understanding the problem of crowdsourcing from a theoretical point of view. Karger et al. \citep{karger14} provide the first theoretical analysis. Under a slightly different probabilistic setting than our paper, they establish the minimax lower bound of the Dawid-Skene model in the regime where the collective wisdom of the crowd is below a threshold. Moreover, they propose a belief propagation algorithm and derive its convergence rate.

We organize the paper as follows. We introduce Dawid-Skene estimator in Section \ref{sec:background}. Our main results are presented in Section \ref{sec:main}, followed by a comparative study of majority voting in Section \ref{sec:compare}. All the proofs are given in the appendix.

\section{Background and Problem Setup} \label{sec:background}

In this section, we provide a precise problem formulation of the paper. We introduce the classical Dawid-Skene model and motivate the estimator. A simple projected EM algorithm is discussed.

\subsection{Probabilistic Model and MLE}

Consider $m$ items. Each of them is associated with a label $y_j^*\in\{0,1\}$. The labels $\{y_j^*\}_{j\in[m]}$, which is called the ground truth, are unknown. In order to infer the ground truth, $n$ workers are hired to assign  $0$ or $1$ to each item. Denote the answer from the $i$-th worker for the $j$-th item as $X_{ij}$. Since the workers may have a variety  of backgrounds, their answers may or may not be accurate. A fundamental question is how to aggregate the workers' answers $\{X_{ij}\}_{i\in[n],j\in[m]}$ and provide a reliable estimator $\{\hat{y}_j\}_{j\in[m]}$ for the ground truth.

To model the workers' abilities, \cite{DawSke79} proposed the so-called confusion matrix. The confusion matrix for the $i$-th worker is denoted as
$$\begin{pmatrix}
\pi_{00}^{(i)} & \pi_{01}^{(i)} \\
\pi_{10}^{(i)} & \pi_{11}^{(i)}
\end{pmatrix}.$$
The number $\pi_{kl}^{(i)}$ stands for the probability for the $i$-th worker to give answer $l$ given the ground truth is $k$. In this paper, we consider a special class of the confusion matrix
$$\begin{pmatrix}
p_i^* & 1-p_i^* \\
1-p_i^* & p_i^*
\end{pmatrix}.$$
Namely, the ability of the $i$-th worker is characterized by the probability of success $p_i^*\in [0,1]$. This is called the one-coin model in the literature of crowdsourcing, because every worker is modeled by a biased coin.

The difficulty of estimating the workers' abilities $p^*=(p^*_1,...,p^*_n)$ is mainly caused by the fact that the ground truth $y^*=(y_1^*,...,y_m^*)$ is unknown. Otherwise, $p_i^*$ can be easily estimated by the frequency of success of the $i$-th worker. \cite{DawSke79} proposed to estimate $p^*$ by maximizing the marginal likelihood function. Given the ground truth, the conditional likelihood is
\begin{equation}
P(X|y,p)=\prod_{j\in[m]}\prod_{i\in[n]}P(X_{ij}|y_j,p_i)=\prod_{j\in[m]}\prod_{i\in[n]} p_i^{\mathbb{I}\{X_{ij}=y_j\}}(1-p_i)^{\mathbb{I}\{X_{ij}=1-y_j\}}. \label{eq:conditionallik}
\end{equation}
Integrating out the ground truth with a uniform prior, the marginal likelihood is
$$P(X|p)=\prod_{j\in[m]}\left(\frac{1}{2}\prod_{i\in[n]}p_i^{X_{ij}}(1-p_i)^{1-X_{ij}}+\frac{1}{2}\prod_{i\in[n]}(1-p_i)^{X_{ij}}p_i^{1-X_{ij}}\right).$$
Then, the maximum likelihood estimator (MLE) is defined as
\begin{equation}
\hat{p}=\arg\max_p \log P(X|p). \label{eq:MLE}
\end{equation}
Note that (\ref{eq:MLE}) is a non-convex optimization problem. We will also discuss an efficient algorithm and provide its statistical error bound.

\subsection{Projected EM Algorithm} \label{sec:DW}

After the MLE $\hat{p}=(\hat{p}_1,...,\hat{p}_n)$ is obtained from (\ref{eq:MLE}), it is natural to plug it into the Bayes formula and get an estimator for the ground truth $y^*$. That is,
\begin{equation}
\hat{y}_j\propto \prod_{i\in[n]}\hat{p}_i^{X_{ij}}(1-\hat{p}_i)^{1-X_{ij}},\quad 1-\hat{y}_j\propto\prod_{i\in[n]}(1-\hat{p}_i)^{X_{ij}}\hat{p}_i^{1-X_{ij}}. \label{eq:Bayes}
\end{equation}
Note that we implicitly use the uniform prior in the Bayes formula and the resulting estimator $\hat{y}$ is a soft label, taking value in $[0,1]^m$. Combining (\ref{eq:MLE}) and (\ref{eq:Bayes}), the pair of estimator $(\hat{p},\hat{y})$ is the global optimizer of the following objective function.
\begin{eqnarray*}
F(p,y) &=& \sum_i\sum_jy_j\Big(X_{ij}\log p_i + (1-X_{ij})\log(1-p_i)\Big) \\
&& + \sum_i\sum_j(1-y_j)\Big(X_{ij}\log(1-p_i)+(1-X_{ij})\log p_i\Big) \\
&& + \sum_j\Big(y_j\log \frac{1}{y_j}+(1-y_j)\log\frac{1}{1-y_j}\Big).
\end{eqnarray*}
It is proved in \cite{NeaHin98} that optimizing over $\log P(X|p)$ is equivalent as optimizing over $F(p,y)$. Namely, define
\begin{equation}
(\hat{p},\hat{y})=\arg\max_{(p,y)}F(p,y). \label{eq:energy}
\end{equation}
Then the definition of $(\hat{p},\hat{y})$ in (\ref{eq:energy}) is equivalent to that in (\ref{eq:MLE}) and (\ref{eq:Bayes}). The form of $F(p,y)$ is more tractable than the likelihood function $\log P(X|p)$. Moreover, it unifies the estimation of $p^*$ and $y^*$ into a single optimization problem.

Note that the objective function is bi-convex in the sense that $F(p,y)$ is both convex with respect to $p$ and with respect to $y$. A natural algorithm for maximizing $F(p,y)$ is to iteratively update $p$ and $y$. Given an initializer $y^{(0)}$, the $t$-th step of the iterative algorithm is
\begin{equation}
p^{(t)}=\arg\max_{p}F(p,y^{(t-1)}),\quad y^{(t)}=\arg\max_y F(p^{(t)},y). \label{eq:iterative}
\end{equation}
Direct calculation gives explicit formulas for (\ref{eq:iterative}),
\begin{eqnarray}
\label{eq:M} p_i^{(t)} &=& \frac{1}{m}\sum_{j\in[m]}\Big((1-X_{ij})\big(1-y_j^{(t-1)}\big)+X_{ij}y_j^{(t-1)}\Big), \\
\nonumber y_j^{(t)} &\propto& \prod_{i\in[n]}\Big(p_i^{(t)}\Big)^{X_{ij}}\Big(1-p_i^{(t)}\Big)^{1-X_{ij}}, \\
\nonumber 1-y_j^{(t)} &\propto&\prod_{i\in[n]}\Big(p_i^{(t)}\Big)^{1-X_{ij}}\Big(1-p_i^{(t)}\Big)^{X_{ij}}.
\end{eqnarray}
This is recognized as the EM algorithm \citep{DemLaiRub77}. In order to achieve good statistical property, we propose to replace the M-step (\ref{eq:M}) by the following projection rule,
\begin{equation}
p_i^{(t)} = \Pi_{\mathcal{C}}\left\{\frac{1}{m}\sum_{j\in[m]}\Big((1-X_{ij})\big(1-y_j^{(t-1)}\big)+X_{ij}y_j^{(t-1)}\Big)\right\}, \label{eq:modifiedM}
\end{equation}
where $\Pi_{\mathcal{C}}$ is the projection operator on a convex set $\mathcal{C}$. We choose $\mathcal{C}$ to be the interval $[\lambda,1-\lambda]$ for some small $\lambda>0$. We refer to the modified EM algorithm as the projected EM.

The motivation for the projection step in (\ref{eq:modifiedM}) is to keep the estimator $p_i^{(t)}$ away from $0$ and $1$. Since once this happens, $p_i^{(t)}$ will be trapped in its current value, which can be a poor local optimizer. Other strategies besides projection include putting a Beta prior on the workers' abilities. We refer to \cite{LiuPenIhl12} for detailed discussion.

Note that the objective function $F(p,y)$ is not jointly convex in $(p,y)$. Thus, the optimization procedure will converge to a local optimum. However, as long as the algorithm is properly initialized,
the iterations after the first step  are in the neighborhood of the ground truth with guaranteed accuracy in probability. To achieve this purpose, a novel initialization step is proposed in this paper for the projected EM to achieve the desired convergence rate. 

The theoretical result for projected EM has an interesting implication on the classical EM. Namely, when the workers' abilities are bounded, they have the same solutions with high probability. We will discuss this implication in Section \ref{sec:EM}.

\section{Main Results} \label{sec:main}

In this section, we present the main results of this paper. Statistical error bounds are established for the projected EM algorithm. We show these estimators converge to the ground truth at exponential rates, and provide  lower bounds to match the exponent up to an absolute constant. To clarify the theoretical setting, all results are stated in probability and expectation conditioning on the truth $(p^*,y^*)$, though  (\ref{eq:MLE}) is defined by the marginal likelihood.

\subsection{Critical Quantities}

In the model specification, each worker's ability is parameterized by $p^*_i$, the probability that he or she gives the correct answer. We are going to define two critical quantities to summarize the wisdom of the crowd of workers. These quantities will appear in the convergence rates.

To motivate the two quantities, observe that there are three kinds of workers: those with $p^*_i>1/2$, $p^*_i=1/2$ and $p^*_i<1/2$. The workers with $p^*_i=1/2$ are spammers and their information is useless. For the workers with $p^*_i<1/2$, though they act in a adversarial style, their information can be used as long as we can detect their behavior and invert the answers. Define the effective ability of the $i$-th worker as
\begin{equation}
\mu_i=p_i^*\mathbb{I}\{p_i^*\geq 1/2\}+(1-p_i^*)\mathbb{I}\{p_i^*<1/2\}. \label{eq:effecmu}
\end{equation}
Note that $\mu_i$ is a number in $[1/2,1]$. Also define $\nu_i=(2\mu_i-1)^2$.

For the crowd of workers, we define the following two quantities to summarize their collective wisdom.
\begin{equation}
\bar{\nu}=\frac{1}{n}\sum_{i\in[n]}\nu_i,\quad \bar{\mu}=\frac{1}{n}\sum_{i\in[n]}\mu_i.\label{eq:wisdom}
\end{equation}
Since both quantities depend on $p_i^*$ through $\mu_i$, they are invariant to the transformation $p_i^*\rightarrow 1-p_i^*$.
Both $\bar{\nu}$ and $\bar{\mu}$ measure the collective wisdom of the crowd. However, they characterize  different perspectives. The quantity $\bar{\nu}$ roughly measures the proportion of experts among the crowd. An example is the following class
\begin{equation}
\mathcal{P}'=\left\{(p_1,...,p_n)\in\{1/2,1\}^n: \sum_{i\in [n]}\mathbb{I}\{p_i=1\}=\ceil{n\bar{\nu}}\right\}, \label{eq:lfq}
\end{equation}
where there are $\ceil{n\bar{\nu}}$ experts with ability $1$ among the $n$ workers and the remaining workers are random guessers. On the other hand, $\bar{\mu}$ measures the absolute ability of the crowd, represented by
\begin{equation}
p'=(\bar{\mu},...,\bar{\mu}).\label{eq:lfmu}
\end{equation}
As we will show in the derivation of the lower bound, (\ref{eq:lfq}) is the least favorable case for deriving the lower bound for the quantity $\bar{\nu}$ when the collective wisdom is relatively low, and (\ref{eq:lfmu}) is the least favorable case  for the quantity $\bar{\mu}$ when the collective wisdom is relatively high. In Section \ref{sec:lower}, we will provide a more detailed discussion for the meanings of $\bar{\nu}$ and $\bar{\mu}$.

\subsection{Error Bounds for Projected EM} \label{sec:subpEM}

In this section, we analyze the proposed projected EM algorithm. Since the objective function $F(p,y)$ is not convex, the algorithm is expected to converge to a local optimizer. We propose a novel initialization step to prevent the algorithm from being stuck in the neighborhood of a bad local optimizer. It ensures that the subsequent 	EM iterations will provide satisfactory statistical accuracy.

\paragraph{Initializer}

Define the numbers $m_0=|\{j\in[m]:y_j^*=0\}|$ and $m_1=|\{j\in[m]:y_j^*=1\}|$. The proportion $m_1/m$ is denoted as $\pi$. The expectation of the numerical average of the $i$-th workers' answers
$M_i=\frac{1}{m}\sum_{j\in[m]}\mathbb{E}X_{ij}$ satisfies the equation
\begin{equation}
M_i = \pi p_i^*+(1-\pi)(1-p_i^*),\quad\text{for each $i\in[n]$}.\label{eq:ini-M_i-eq}
\end{equation}
Replacing $M_i$ by its empirical version $\hat{M}_i=\frac{1}{m}\sum_{j\in[m]}X_{ij}$, we obtain a natural estimator
\begin{equation}
\hat{p}_i = \frac{\hat{M}_i-(1-\pi)}{2\pi-1},\quad\text{for each $i\in[n]$}. \label{eq:ini-pi}
\end{equation}
With some positive $\bar{\lambda}$ to be specified later, define the projected worker's ability estimator by
\begin{equation}
p_i^{(0)}=\bar{\lambda}\mathbb{I}\{\hat{p}_i<\bar{\lambda}\}+\hat{p}_i\mathbb{I}\{\hat{p}_i\in[\bar{\lambda},1-\bar{\lambda}]\}+(1-\bar{\lambda})\mathbb{I}\{\hat{p}_i>1-\bar{\lambda}\}. \label{eq:ini-p0}
\end{equation}
The initializer $y^{(0)}$ is determined by
\begin{equation}
y_j^{(0)} \propto \prod_{i\in[n]}\Big(p_i^{(0)}\Big)^{X_{ij}}\Big(1-p_i^{(0)}\Big)^{1-X_{ij}},\quad 1-y_j^{(0)}\propto \prod_{i\in[n]}\Big(p_i^{(0)}\Big)^{1-X_{ij}}\Big(1-p_i^{(0)}\Big)^{X_{ij}}. \label{eq:ini-def}
\end{equation}
Note that the above initializer (\ref{eq:ini-def}) requires knowledge of $\pi$, which is not applicable in practice. An estimator of $\pi$ is inspired by the following proposition.
\begin{proposition}\label{prop:equationforpi}
Define $Q_j=\frac{1}{n}\sum_{i\in[n]}X_{ij}$ for each $j\in[m]$. The quantity $\pi$ is a solution to the following equation,
\begin{equation}
\pi^2-\pi+\frac{\frac{1}{2m^2}\sum_{jk}\mathbb{E}(Q_j-Q_k)^2}{\frac{4}{m}\sum_{j\in[m]}\mathbb{E}(Q_j-1/2)^2}=0. \label{eq:clear}
\end{equation}
\end{proposition}
According to the above proposition, it is natural to
estimate $\pi$ by the empirical version of (\ref{eq:clear}). That is, we define $\hat{\pi}$ to be the solution to
\begin{equation}
\hat{\pi}^2-\hat{\pi}+\frac{\frac{1}{2m^2}\sum_{jk}(Q_j-Q_k)^2}{\frac{4}{m}\sum_{j\in[m]}(Q_j-1/2)^2}=0. \label{eq:mmequation}
\end{equation}
It is worthwhile to note that the equation (\ref{eq:mmequation}) has two solutions. The coefficients of (\ref{eq:mmequation}) implies that the sum of the two solutions is $1$. In other words, if $\hat{\pi}$ is one solution, then $1-\hat{\pi}$ is the other. Both $\hat{\pi}$ and $1-\hat{\pi}$ can be used to calculate (\ref{eq:ini-pi}), (\ref{eq:ini-p0}) and (\ref{eq:ini-def}). Direct calculation leads to the fact that $\hat{p}_i(\hat{\pi})+\hat{p}_i(1-\hat{\pi})=1$, $p_i^{(0)}(\hat{\pi})+p_i^{(0)}(1-\hat{\pi})=1$ and $y_j^{(0)}(\hat{\pi})+y_j^{(0)}(1-\hat{\pi})=1$.

\paragraph{Projected EM}

The remark after Proposition \ref{prop:equationforpi} suggests that both of the two solutions to the equation (\ref{eq:mmequation}) can be used to compute the initializer $y^{(0)}$. Plugging $y^{(0)}$ into the projected EM algorithm described in Section \ref{sec:DW}, we obtain an estimator $y^{(t)}$ for some $t\geq 1$ converging to either $y^*$ or $1-y^*$. Whether $y^{(t)}$ converges to $y^*$ or $1-y^*$ depends on the choice of $\hat{\pi}$ or $1-\hat{\pi}$.

To this end, we propose a method to distinguish $y^*$ and $1-y^*$. Given $y^{(t)}$ for some $t\geq 1$, define
\begin{equation}
\check{p}_i^{(t+1)}=\frac{1}{m}\sum_{j\in[m]}\Big((1-X_{ij})y_j^{(t)}+X_{ij}y_j^{(t)}\Big),\quad \text{for }i\in [n],\label{eq:thisisthekey}
\end{equation}
which is identical to the subsequent M-step without projection. Then, define
\begin{eqnarray}
\label{eq:pEMty}\tilde{y}_j &=& \begin{cases}
y_j^{(t)}, &\frac{1}{n}\sum_{i\in[n]}\check{p}_i^{(t+1)}> \frac{1}{2},\\
1-y_j^{(t)},&\frac{1}{n}\sum_{i\in[n]}\check{p}_i^{(t+1)}\leq \frac{1}{2},
\end{cases} \\
\label{eq:pEMtp}\tilde{p}_i &=& \begin{cases}
\check{p}_i^{(t+1)}, &\frac{1}{n}\sum_{i\in[n]}\check{p}_i^{(t+1)}> \frac{1}{2},\\
1-\check{p}_i^{(t+1)},&\frac{1}{n}\sum_{i\in[n]}\check{p}_i^{(t+1)}\leq \frac{1}{2},
\end{cases}
\end{eqnarray}
for all $i\in[n]$ and $j\in[m]$.

\paragraph{Theoretical Results}

Now we present a statistical analysis for the projected EM algorithm discussed in Section \ref{sec:DW} with the proposed initializer. A general theory covers any initializer $y^{(0)}$ whose labeling error is below a certain rate is referred to
 Theorem \ref{thm:pEMgen} in the appendix. Before stating the main result, we need to define a new quantity. Recall that the effective workers' abilities $\mu_1,...,\mu_n$ defined in (\ref{eq:effecmu}). We define the effective average ability with respect to the projected EM algorithm as
$$
\bar{\mu}_{\lambda}=\frac{1}{n}\left(\sum_{\{i:\mu_i\leq 1-\lambda\}}\mu_i + \sum_{\{i:\mu_i>1-\lambda\}}(1-\lambda)\right), $$
where $\lambda$ is the tuning parameter in the projection step (\ref{eq:modifiedM}). When $\lambda=0$, we have $\bar{\mu}_{\lambda}=\bar{\mu}$. For any $a,b\in [0,1]$, we define the Kullback-Leibler (KL) divergence by $D(a||b)=a\log\frac{a}{b}+(1-a)\log\frac{1-a}{1-b}$.

\begin{thm}[Projected EM: Labels] \label{thm:pEM}
Let $y^{(t)}$ be the sequence of the projected EM algorithm with the proposed initialization. Assume  for sufficiently large $n$ and $m$, $n^2\log m\leq m\leq e^n$, $|2\pi-1|\geq c$ for some constant $c>0$ and $\frac{\log m+\log n}{n}\leq \bar{\nu}\leq 1-n\sqrt{\frac{\log m}{m}}$. Whenever the tuning parameters are chosen in the range
\begin{equation}
\bar{\lambda}=\frac{1}{6},\quad 16\bar{\nu}^{-1}\sqrt{\frac{\log m}{m}}\leq\lambda\leq\frac{1}{8}-\frac{1}{2}\sqrt{\frac{\log m}{m}}, \label{eq:asstuning}
\end{equation}
then for any $y^*\in\{0,1\}^m$ and $t\geq 1$, we have
$$\min\left\{\frac{1}{m}\sum_{j\in[m]}|y_j^{(t)}-y_j^*|,\frac{1}{m}\sum_{j\in[m]}|y_j^{(t)}-(1-y_j^*)|\right\}\leq \exp\Bigg(-\frac{1}{2}n\max\Big(\bar{\nu}, D(\bar{\mu}_{\lambda}||1-\bar{\mu}_{\lambda})\Big)\Bigg),$$
uniformly over all $t\geq 1$
with probability at least $1-C'/m$ for some positive constant $C'$. Moreover, if we further assume 
\begin{equation}
\frac{1}{n}\sum_{i\in[n]}p_i^*>\frac{1}{2}+2\sqrt{\frac{\log m}{nm}}, \label{eq:asspave}
\end{equation}
for $\tilde{y}$ defined as (\ref{eq:pEMty}) with any $t\geq 1$,
$$\frac{1}{m}\sum_{j\in[m]}|\tilde{y}_j-y_j^*|\leq \exp\Bigg(-\frac{1}{2}n\max\Big(\bar{\nu}, D(\bar{\mu}_{\lambda}||1-\bar{\mu}_{\lambda})\Big)\Bigg),$$
uniformly over all $t\geq 1$
with probability at least $1-C'/m$.
\end{thm}

Instead of analyzing the convergence of the projected EM algorithm from an optimization view, we study its statistical property at each iteration. It is interesting that there is a simple algorithm that has a guaranteed statistical performance. It is well known that an EM-type algorithm converges to a local optimizer \citep{wu1983convergence,NeaHin98,xu1996convergence}. We overcome this issue by  modifying the M-step and starting from a good initializer.  We have a few remarks regarding the theorem.

\begin{enumerate}
\item \textit{Remark on the assumption.}
The assumption $|2\pi-1|>c$ is critical for the initialization step to work. When it is not satisfied, the equation (\ref{eq:ini-pi}) becomes degenerate and the initialization will fail.

\item \textit{Remark on the tuning parameter.}
The exponent $D(\bar{\mu}_{\lambda}||1-\bar{\mu}_{\lambda})$ increases to $D(\bar{\mu}||1-\bar{\mu})$ as $\lambda$ decreases to $0$. Thus, a smaller $\lambda$ leads to faster convergence rate.  When $\bar{\nu}$ is at a constant level, the optimal tuning parameter scales as $\lambda\asymp \sqrt{\frac{\log m}{m}}$, according to (\ref{eq:asstuning}). The tuning parameter $\bar{\lambda}$ in (\ref{eq:ini-p0}) is set to $1/6$ for simplicity. Its range can be extended to a small interval without affecting the theoretical results.

\item \textit{Remark on the exponent.}
An important contribution of this paper is to reveal the exponent $\max\left\{\bar{\nu},D(\bar{\mu}_{\lambda}||1-\bar{\mu}_{\lambda})\right\}$ in the rate of convergence. For any $p^*$, the quantity $D(\bar{\mu}_{\lambda}||1-\bar{\mu}_{\lambda})$ is bounded by $D(\lambda||1-\lambda)$. When $p_i^*\in [\lambda,1-\lambda]$ for all $i\in[n]$, we have $D(\bar{\mu}||1-\bar{\mu})=D(\bar{\mu}_{\lambda}||1-\bar{\mu}_{\lambda})$. In this case, the exponent is $\max\left\{\bar{\nu},D(\bar{\mu}||1-\bar{\mu})\right\}$.
To further understand the meaning of this exponent, we provide two examples. In the first example, we consider $p_i^*\in\{0.5,0.9\}$ for each $i\in[n]$. Letting $\bar{\nu}= n^{\delta-1}$, for some constant $\delta\in (0,1)$. Then it is easy to see that $\bar{\mu}=\frac{1}{2}1+\frac{5}{8}n^{\delta-1}$, which implies $D(\bar{\mu}||1-\bar{\mu})\asymp n^{2(\delta-1)}$. Hence, $\bar{\nu}\gg D(\bar{\mu}||1-\bar{\mu})$
for sufficiently large $n$. In the second example, let us consider $p_i^*=\bar{\mu}=1-\xi_n$ for each $i\in[n]$, where $\xi_n$ is any positive sequence decreasing to $0$ as $n$ increases. Then $D(\bar{\mu}||1-\bar{\mu})$ increases to $\infty$, but $\bar{\nu}$ is always bounded by $1$. We have $\bar{\nu}\ll D(\bar{\mu}||1-\bar{\mu})$. Generally speaking, when the workers' abilities are heterogeneous, $\bar{\nu}$ is the dominating exponent. When the workers' abilities are homogeneous, $D(\bar{\mu}||1-\bar{\mu})$ is the dominating exponent.
We are going to provide lower bounds in Section \ref{sec:lower} to show that this exponent is necessary.

\item \textit{Remark on the convergence.} The theorem provides exponential bound which is uniform over the sequence of estimators $\{y^{(t)}\}_{t\geq 1}$. In particular, we may choose $y^{(1)}$, the one-step estimator as the final estimator for the ground truth. In practice, a few more iterations may help to further reduce some error, but not be able to improve the rate on the exponent.

\item \textit{Remark on the estimators.} The estimator is a soft label estimator. To get a hard label estimator, an obvious way is to use the indicator $\mathbb{I}\{\hat{y}_j\geq 1/2\}$. The clustering and labeling error rate of the hard label estimator follows directly from Theorem \ref{thm:pEM} by the simple inequalities
$$\left|\mathbb{I}\{\hat{y}_j\geq 1/2\}-1\right|\leq 2|\hat{y}_j-1|,\quad\text{and}\quad \left|\mathbb{I}\{\hat{y}_j\geq 1/2\}-0\right|\leq 2|\hat{y}_j-0|.$$
The extra factor $2$ does not affect the exponent.

\end{enumerate}

Besides the ground truth estimator, the algorithm also produces the workers' abilities estimator $p^{(t)}$. However, since each $p_i^{(t)}$ is projected into the interval $[\lambda,1-\lambda]$, it may incur bias when the true workers' abilities are large. This issue can be easily fixed. While the projected $p^{(t)}$ is used in the EM iteration, we do not need projection in the final step to estimate $p^*$. To be specific, we may use (\ref{eq:pEMtp}) as the estimator for the workers' abilities. The next theorem gives non-asymptotic error bounds for the workers' abilities estimator.
Moreover, we also derive the exact asymptotic distribution. Before stating the result, we need some new notation. Define the Fisher's information matrix $J=(J_{ij})$ with $J_{ii}^{-1}=p_i^*(1-p_i^*)$ for all $i\in[n]$ and $J_{ij}=0$ for all $i\neq j$. For a subset $S\subset[n]$ and a vector $v\in\mathbb{R}^n$, let $v_S=(v_i)_{i\in S}$ be the sub-vector. We also let $J_{S}$ be the $S\times S$ sub-matrix of $J$ and $I_{|S|}$ be the $|S|\times |S|$ identity matrix. The symbol $\xrightarrow{\mathcal{D}}$ stands for convergence in distribution.

\begin{thm}[Projected EM: Workers' Abilities] \label{thm:pEMwa}
Consider $\tilde{p}$ defined in (\ref{eq:pEMtp}) for some $t\geq 1$ from the projected EM algorithm with the proposed initialization. Under the settings of Theorem \ref{thm:pEM}, 
\begin{eqnarray}
\label{eq:linfty} \left\|\tilde{p}-p^*\right\|_{\infty} &\leq& 2\sqrt{\frac{\log m}{m}}, \\
\label{eq:l2} \frac{1}{n}\left\|\tilde{p}-p^*\right\|^2 &\leq& \frac{1}{nm}\sum_{i\in[n]}p_i^*(1-p_i^*) + C\sqrt{\frac{\log m}{nm^2}},
\end{eqnarray}
with probability at least $1-C'/m$. Furthermore, if
\begin{equation}
\max_{i\in[n]}\frac{1}{p_i^*(1-p_i^*)}\leq C_1, \label{eq:assboundp}
\end{equation}
for some positive constant $C_1$.
then for any subset $S\subset[n]$ with bounded cardinality $|S|\leq C_2$ for some positive constant $C_2$, we have the central limit theorem,
\begin{equation}
\sqrt{m} J_S^{1/2}\Big(\tilde{p}_S-p_S^*\Big)\xrightarrow{\mathcal{D}} N\Big(0, I_{|S|}\Big),\quad \text{as }n,m\rightarrow\infty. \label{eq:CLT}
\end{equation}
Moreover, we also have the high-dimensional central limit theorem for the whole vector,
\begin{equation}
\sup_{t\in\mathbb{R}}\left|\mathbb{P}\Bigg(\max_{i\in[n]}\left|\frac{\sqrt{m}(\tilde{p}_i-p_i^*)}{\sqrt{p_i^*(1-p_i^*)}}\right|\leq t\Bigg)-\mathbb{P}\Big(\max_{i\in[n]}|Z_i|\leq t\Big)\right|\leq \frac{C\log n}{m^{1/8}}, \label{eq:hdCLT}
\end{equation}
where $\{Z_i\}_{i\in [n]}$ are i.i.d. $N(0,1)$ and $C$ is a positive constant.
\end{thm}

The four convergence results in Theorem \ref{thm:pEMwa} characterize the accuracy for estimating workers' abilities from different perspectives. The high-probability non-asymptotic results (\ref{eq:linfty}) and (\ref{eq:l2}) bound worst-case deviation and average-case deviation of $\tilde{p}$ from the truth $p^*$. The bound in (\ref{eq:l2}) has two parts, where the first part is roughly the mean of $\frac{1}{n}\left\|\tilde{p}-p^*\right\|^2$ in the sense that
$$\mathbb{E}\left(\frac{1}{n}\left\|\tilde{p}-p^*\right\|^2\right)\approx \frac{1}{nm}\sum_{i\in[n]}p_i^*(1-p_i^*),$$
and the second part is the deviation of $\frac{1}{n}\left\|\tilde{p}-p^*\right\|^2$  from its mean. For the asymptotic result, we need the extra assumption (\ref{eq:assboundp}) to prevent  $p_i^*$ from being extremely close to $0$ or $1$, in which case $\tilde{p}_i$ is more or less deterministic because of the tiny variance. The result (\ref{eq:CLT}) implies the classical asymptotic efficiency of Fisher by observing that the asymptotic covariance taken as the inverse information matrix $J_S^{-1}$. While (\ref{eq:CLT}) focuses on asymptotics with finite dimension, the result of (\ref{eq:hdCLT}) establishes the high-dimensional asymptotics with growing dimension $n$. The right side of (\ref{eq:hdCLT}) is easily going to $0$ as long as the dimension is not exponentially large compared to $m^{1/8}$. An interesting consequence of (\ref{eq:CLT}) and (\ref{eq:hdCLT}) is that $\tilde{p}_1,...,\tilde{p}_n$ behave independently in the limit as $(n,m)\rightarrow\infty$. The phenomenon is intuitive because when the numbers of workers and items are large, estimation of the ability of the $i$-th worker mainly depends on her own performance.

\subsection{Lower Bounds} \label{sec:lower}

In this section, we show that the two exponents $\bar{\nu}$ and $D(\bar{\mu}||1-\bar{\mu})$ cannot be improved under the Dawid-Skene model. The main conclusion is that when $\bar{\nu}<1/2$, $\exp(-Cn\bar{\nu})$ is the minimax lower bound for the labeling error rate, and when $\bar{\nu}\geq 1/2$, $\exp\Big(-CnD(\bar{\mu}||1-\bar{\mu})\Big)$ is the minimax lower bound, where $C>0$ is some absolute constant in both cases.

To rigorously state the result, we define two parameter spaces.
\begin{eqnarray}
\label{eq:paraspacenu} \mathcal{P}_{\bar{\nu}} &=& \left\{(p_1,...,p_n)\in [0,1]^n: \frac{1}{n}\sum_{i\in[n]}(2p_i-1)^2=\bar{\nu}\right\}, \\
\label{eq:paraspacemu} \mathcal{P}_{\bar{\mu}} &=& \left\{(p_1,...,p_n)\in[0,1]^n: \frac{1}{n}\sum_{i\in[n]}\mu(p_i)=\bar{\mu}\right\},
\end{eqnarray}
where $\mu(\cdot)$ is defined as $\mu(p_i)=p_i\mathbb{I}\{p_i\geq 1/2\}+(1-p_i)\mathbb{I}\{p_i<1/2\}$. Recall that $\mu_i=\mu(p_i^*)$ in (\ref{eq:effecmu}). The first parameter space $\mathcal{P}_{\bar{\nu}}$ collects all workers' abilities with the same $\bar{\nu}$-value and the second parameter space $\mathcal{P}_{\bar{\mu}}$ collects all workers' abilities with the same $\bar{\mu}$-value. We use the notation $P_{p,y}$ to stands for the probability distribution
$$P_{p,y}(X)=\prod_{j\in[m]}\prod_{i\in[n]} p_i^{\mathbb{I}\{X_{ij}=y_j\}}(1-p_i)^{\mathbb{I}\{X_{ij}=1-y_j\}},$$
which is actually the conditional likelihood (\ref{eq:conditionallik}) in the Dawid-Skene model. We denote the expectation with respect to $P_{p,y}$ as $E_{p,y}$.

In the first regime, the workers' abilities are low, characterized by $\bar{\nu}<1/2$. The difficulty of the problem is characterized by $\mathcal{P}_{\bar{\nu}}$.
\begin{thm} \label{thm:lower1}
Assume $\bar{\nu}<1/2$. Then for any estimator $\hat{y}\in[0,1]^m$,
$$\sup_{y\in\{0,1\}^m,p\in\mathcal{P}_{\bar{\nu}}}E_{p,y}\left(\frac{1}{m}\sum_{j\in[m]}|\hat{y}_j-y_j|\right)\geq \frac{1}{8(6e)^2}\exp\Big(-6n\bar{\nu}\Big),$$
for every $n\geq 4$ and $m\geq 1$.
\end{thm}
In the second regime, the workers' abilities are high, characterized by $\bar{\mu}\geq 3/4$. The difficulty of the problem is characterized by $\mathcal{P}_{\bar{\mu}}$.
\begin{thm} \label{thm:lower2}
Assume $\bar{\mu}\geq 3/4$. Then for any estimator $\hat{y}\in[0,1]^m$,
$$\sup_{y\in\{0,1\}^m,p\in\mathcal{P}_{\bar{\mu}}}E_{p,y}\left(\frac{1}{m}\sum_{j\in[m]}|\hat{y}_j-y_j|\right)\geq \frac{1}{8}\exp\Big(-8nD(\bar{\mu}||1-\bar{\mu})\Big),$$
for every $n\geq 6$ and $m\geq 1$.
\end{thm}

We remark that the conditions $\bar{\nu}<1/2$ and $\bar{\mu}\geq 3/4$ cover all cases. Consider $\bar{\nu}$ and $\bar{\mu}$ defined by the same vector $p=(p_1,...,p_n)$. Then, we have
$$\bar{\nu}=(2\bar{\mu}-1)^2+\frac{4}{n}\sum_{i\in[n]}(\mu(p_i)-\bar{\mu})^2.$$
Since $\mu(p_i)\in [1/2,1]$ for each $i\in[n]$, $\frac{4}{n}\sum_{i\in[n]}(\mu(p_i)-\bar{\mu})^2\leq \frac{1}{4}$. Whenever $\bar{\nu}<1/2$ is not satisfied, we have $\bar{\nu}\geq 1/2$, which implies
$$(2\bar{\mu}-1)^2\geq \bar{\nu}-\frac{1}{4}=\frac{1}{4}.$$
This leads to $\bar{\mu}\geq 3/4$.

In order to understand the meaning of $\bar{\nu}$ and $\bar{\mu}$, let us briefly discuss the main idea of the proof. The standard technique for proving minimax lower bound is to find a least favorable subset of the parameter space. The subset needs to be hard enough to carry the difficulty of the problem and it also needs to be easy enough to calculate the minimax risk. See \cite{yang1999information}, \cite{tsybakov09} and \cite{yu1997assouad} for more details of lower bound techniques.

In Theorem \ref{thm:lower1}, the subset we construct is (\ref{eq:lfq}). Notice $\mathcal{P}'$ defined in (\ref{eq:lfq}) may not be a subset of $\mathcal{P}_{\bar{\nu}}$ when $n\bar{\nu}$ is not a integer. However, we have rounded $n\bar{\nu}$ to $\ceil{n\bar{\nu}}$ so that $\mathcal{P}'$ is a slightly easier problem than $\mathcal{P}_{\bar{\nu}}$. Its minimax risk lower bounds the minimax risk of $\mathcal{P}_{\bar{\nu}}$. The parameter $\bar{\nu}$ in $\mathcal{P}'$ has a specific meaning that the proportion of experts is about $\bar{\nu}$ and the proportion of spammers is about $1-\bar{\nu}$. Intuitively speaking, since we don't know which one is expert or spammer, we have to make an error with exponent proportional to $-n\bar{\nu}$.

In Theorem \ref{thm:lower2}, the subset we construct is the singleton (\ref{eq:lfmu}). It is clear that in this case, $\bar{\mu}$ characterizes the (average) ability of the workers. For each item $j\in[m]$, the problem is reduced to the testing between $y_j=0$ and $y_j=1$. The well-known Chernorff information bound (Chapter 11 of \cite{cover12}) gives the best exponent of the error proportional to $-nD(\bar{\mu}||1-\bar{\mu})$ asymptotically. We use a different proof to obtain a non-asymptotic lower bound which is valid for every $n\geq 6$ and $m\geq 1$. The non-asymptotic lower bound is more general because it includes the case where $\bar{\mu}$ may depend on $n$ and $m$. As a price, our constant before the exponent is slightly looser than the one in the asymptotics.

A similar lower bound argument as Theorem \ref{thm:lower1} has also been established by Karger et al. \citep{karger14}. They considered the Dawid-Skene model with model parameters $(p_1,...,p_n)$ i.i.d. drawn from some distribution satisfying $\mathbb{E}(2p-1)^2=\bar{\nu}$ and established a lower bound with exponent proportional to $-n\bar{\nu}$. In contrast, we consider fixed model parameters that lead to a different proof strategy than theirs.

\section{Comparison with Majority Voting} \label{sec:compare}

In this section, we present a comparative study between the Dawid-Skene estimator and the majority voting estimator. The majority voting method is the simplest crowdsourcing algorithm, defined as
$$\hat{y}_j=\mathbb{I}\left\{\sum_{i\in[n]}X_{ij}\geq \frac{n}{2}\right\},\quad \text{for all }j\in[m].$$
It estimates the ground truth by aggregating the results from each worker with equal weights.

A significant feature of majority voting is that it is not derived from a specific model, while Dawid-Skene estimator is derived from a particular model assumption. We expect that majority voting is more robust to  model misspecification, but is inferior to Dawid-Skene estimator when the model is well-specified.  We are going to illustrate this point by two examples.

In the first example, we assume the one-coin model setting in this paper, but let the proportion of spammers be very high. In this setting, majority voting performs poorly because the information provided by the experts are washed out by the majority, who are spammers. On the other hand, Dawid-Skene estimator takes the advantage of the model assumption and is able to identify the experts.

In the second example, we assume a misspecified model, where there are two types of items. In this setting, while majority voting is quite robust to the model assumption, Dawid-Skene estimator is inferior because it puts most weights on workers who are experts only in one type of the item and thus fails to correctly labels the other type.

In addition to the comparison study presented in this section,  a model-based weighted majority voting is shown to enjoy certain optimality properties \citep{berend2014consistency} when the true labels are assumed to be known.

\subsection{Well-Specified Model: Dawid-Skene Is Better}

Suppose there are $n$ workers. Assume that only $\ceil{n^{\delta}}$ of them are experts and all the other workers are spammers. We let $\delta\in (0,1)$ so that the proportion of experts among the crowd is
$$\frac{\ceil{n^{\delta}}}{n}=O(n^{-(1-\delta)})=o(1).$$
Let the experts have ability $p_i^*=1$ and spammers have ability $p_i^*=1/2$. Under this setting, the performance of the majority voting is characterized by the following theorem.

\begin{thm}\label{thm:votebad}
For the majority voting $\hat{y}$, the expected error rate has the following dependence on $\delta$. When $\delta\in (1/2,1)$,
$$\frac{1}{m}\sum_{j\in[m]}\mathbb{E}|\hat{y}_j- y_j^*|=o(1).$$
When $\delta=1/2$,
$$\frac{1}{m}\sum_{j\in[m]}\mathbb{E}|\hat{y}_j- y_j^*|=\Phi(-1)+o(1),$$
where $\Phi$ is the cumulative distribution function of $N(0,1)$. When $\delta\in (0,1/2)$,
$$\frac{1}{m}\sum_{j\in[m]}\mathbb{E}|\hat{y}_j- y_j^*|=\frac{1}{2}+o(1).$$
The symbol $o(1)$ means whatever converges to $0$ as $n\rightarrow\infty$.
\end{thm}

The theorem finds an interesting phase transition phenomenon for majority voting. That is, majority voting is consistent if and only if $\delta>1/2$. When $\delta<1/2$, majority voting behaves like random guess.

On the other hand, the performance of Dawid-Skene estimator is guaranteed by Theorem \ref{thm:pEM}. In the current setting, the quantity $\bar{\nu}$ dominates $D(\bar{\mu}||1-\bar{\mu})$ in the exponent.  Therefore, for Dawid-Skene estimator, we have
$$\frac{1}{m}\sum_{j\in[m]}|\hat{y}_j-y_j^*|\leq\exp\Big(-Cn^{\delta}\Big),$$
with probability at least $1-C'/m$ for sufficiently large $n$ and $m$. The convergence rate is still exponentially fast.

\subsection{Misspecified Model: Majority Voting Is Better}

Now we present a setting where the model is misspecified. This example was communicated to us by Nihar Shah\footnote{This example was constructed by Nihar Shah from UC Berkeley when he was an intern at Microsoft Research.}. Suppose there are $n$ workers and $m$ items. Among the $m$ items, there are $m_1$ items of Type I and $m_2$ items of Type II and $m_1+m_2=m$. We write the set of Type-I items $S_1$ and the set of Type-II items $S_2$. The $n$ workers are also divided into two groups $G_1$ and $G_2$. The first group has $n_1$ workers and they are experts on Type-I items but do not have knowledge on Type-II items. The second group has $n_2$ workers and they are experts on Type-II items but do not have knowledge on Type-I items. We also have $n_1+n_2=n$.

Let $T_{ij}$ be the Bernoulli random variable indicating that the $i$-th worker correctly label the $j$-th item. We consider the model
$$\mathbb{P}(T_{ij}=1)=\frac{4}{5},\quad\text{for }i\in G_1, j\in S_1,\quad\mathbb{P}(T_{ij}=1)=\frac{1}{2},\quad\text{for }i\in G_1, j\in S_2,$$
$$\mathbb{P}(T_{ij}=1)=\frac{1}{2},\quad\text{for }i\in G_2, j\in S_1,\quad\mathbb{P}(T_{ij}=1)=\frac{4}{5},\quad\text{for }i\in G_2, j\in S_2.$$
Then, given the ground truth $y^*$, the data we observe is generated by
\begin{equation}
X_{ij}=T_{ij}y_j^*+(1-T_{ij})(1-y_j^*),\quad\text{for all }i\in[m],j\in[m].\label{eq:repres}
\end{equation}

The failure of Dawid-Skene is characterized by the following theorem.
\begin{thm} \label{thm:MLEbad}
Let $\hat{y}=y^{(t)}$ be the sequence of the projected EM algorithm. Let $n_1=\ceil{n/2}$ and $m_2=\ceil{m^{1/2}}$. As long as for sufficiently large $n,m$, $\frac{n\log m}{m}$ is sufficiently small, we have,
$$\mathbb{P}\left(\frac{1}{m}\sum_{j\in[m]}|\hat{y}_j-y_j^*|\geq\frac{1}{8}m^{-1/2}\right)\geq 0.3.$$
\end{thm}

Theorem \ref{thm:MLEbad} says that the error rate of Dawid-Skene cannot converge faster than a polynomial of $m$. An exponential rate is impossible for Dawid-Skene in this case. In contrast, majority voting converges exponentially fast.

\begin{thm} \label{thm:votinggood}
Let $\hat{y}$ be majority voting. In the same setting of Theorem \ref{thm:MLEbad}, we have for any $y^*\in\{0,1\}^m$,
$$\frac{1}{m}\sum_{j\in[m]}|\hat{y}_j- y_j^*|\leq \exp\Bigg(-\frac{1}{25}n\Bigg),$$
with probability at least $1-e^{-n/200}$.
\end{thm}

The purpose of this example is to show that majority voting is less sensitive to the model assumption. In the literature of crowdsourcing, there are some extensions of the Dawid-Skene model. For example, \cite{zhoplaby12} and \cite{LiuPenIhl12}. However, the example constructed here is not included in any of the extension. Though  using a more general confusion matrix is helpful to model items from more than one types, the main difficulty, however, is that we do not know the type of workers, either. In order to get a good labeling of the items, we have to estimate the workers' types as well. One possible solution is to introduce another set of latent variables $\{z_i^*\}_{i\in[n]}\in\{0,1\}^n$ for the workers to indicate their expertise. Then in the E-step of the EM or projected EM algorithm, we need to update both $y^{(t)}$, the item label estimator, and $z^{(t)}$, the worker type estimator. Whether such generalization has any theoretical guarantee is an open problem and will be considered in the future research.

\section{Discussion} \label{sec:disc}

\subsection{Related Work and Future Directions}

The work of Dawid and Skene \cite{DawSke79} laid a solid foundation in the field of crowdsourcing.  Extensions of the framework under a Bayesian setting were investigated by \cite{RayYuZha10,LiuPenIhl12,CheLinZho13}.

The model of Dawid and Skene implicitly assumes that a worker performs equally well across all items in a common class.  In practice, however, it is often the case that one item is more difficult to label than another. To address this heterogeneous issue, Zhou et al. \cite{zhoplaby12} propose a minimax entropy principle for crowdsourcing. The observed labels are modeled jointly by the worker confusion matrices and item confusion vectors through an exponential family model. Moreover, it turns out that the probabilistic model can be equivalently derived from a natural assumption of objective measurements of worker ability and item difficulty. Such  objectivity arguments have been widely discussed in the literature of mental test theory \cite{Ras61,LorNov68}.

Though the framework of Dawid and Skene has been widely used and well extended in crowdsourcing on the algorithmic side, there has been no theoretical work addressing convergence and optimality issues under the Dawid-Skene setting. To the best of our knowledge, the only exception is the work of Karger et al. \cite{karger14}. They proposed a belief propagation algorithm using a Haldane prior\footnote{A Haldane prior assumes each worker's ability is either $1/2$ or $1$ with equal probabilities.} on workers' abilities  and derive its rate of convergence under the one-coin model. They essentially reveal the exponent $-n\bar{v}$ in the rate of convergence under the assumption that workers' abilities are bounded, while in our work both $-n\bar{v}$ and $-nD(\bar{\mu}||1-\bar{\mu})$ play important roles on the exponent in various regimes. In addition, they consider the question of task assignment which is not addressed in our paper.

In addition to showing that Dawid-Skene estimator achieves the minimax rate, we are also interested in studying whether its generalization also shares  optimality. For example, what if we consider both worker confusion matrices and item confusion vectors \cite{Ras61,zhoplaby12}. The technique used in this paper cannot be directly extended to that setting. We will consider this harder problem in our future work.

\subsection{The Classical EM Algorithm} \label{sec:EM}

We have shown in this paper that the projected EM  enjoys nearly optimal exponential convergence rate. It is curious whether the classical EM without the projection step also has such statistical property. When the workers' abilities are  bounded away from $0$ and $1$, this is indeed the case. The result is implied from the following observation of the two algorithms.

\begin{thm} \label{thm:EM}
Let $y^{(t)}$ be the sequence of the projected EM algorithm and let $\check{y}^{(t)}$ be the sequence of the EM algorithm. Both use the proposed initialization step. Under the assumptions of Theorem \ref{thm:pEM} and further assume that
\begin{equation}
p_i^*\in\Bigg[\lambda+2\sqrt{\frac{\log m}{m}},1-\lambda-2\sqrt{\frac{\log m}{m}}\Bigg],\quad\text{for all }i\in[n]. \label{eq:boundedability}
\end{equation}
Then we have
$$y^{(t)}=\check{y}^{(t)},\quad \text{for all }t\geq 1,$$
with probability at least $1-C'/m$.
\end{thm}

Theorem \ref{thm:EM} says when the workers' abilities are not at extreme, the projected EM and the EM have the same iterations. Since the assumption (\ref{eq:boundedability}) is usually satisfied for real data where workers' abilities are low (close to $1/2$), it explains why in practice  the EM behaves well.

\appendix

\section{Proofs of Theorem \ref{thm:pEM} and Theorem \ref{thm:EM}} \label{sec:proofpEM}

This section presents technical proofs related to the projected EM algorithm. In Section \ref{sec:genpEM}, we state and prove a result of error bounds for projected EM with general initializations satisfying a certain condition. Then, in Section \ref{sec:inipEM}, we show that the initialization step proposed in Section \ref{sec:subpEM} satisfies the condition. The proof of Proposition \ref{prop:equationforpi}, which is the key of the proposed initialization step, is also given in Section \ref{sec:inipEM}. Finally, we prove Theorem \ref{thm:pEM} and Theorem \ref{thm:EM} as corollaries in Section \ref{sec:finalpEM}.

Let us first introduce some technical lemmas which will be used in the proof.

\begin{lemma}[Hoeffding's Inequality] \label{lem:hoeffding}
For independent bounded random variables $\{X_i\}_{i\in[n]}$ satisfying $X_i\in [a_i,b_i]$ for all $i\in[n]$, we have
$$\mathbb{P}\left(\left|\frac{1}{n}\sum_{i\in[n]}(X_i-\mathbb{E}X_i)\right|>t\right)\leq 2\exp\Bigg(-\frac{2n^2t^2}{\sum_{i\in[n]}(b_i-a_i)^2}\Bigg),$$
for any $t\geq 0$.
\end{lemma}

For a sub-exponential random variable $X$, define its sub-exponential norm as
$$\Norm{X}_{\psi_1}=\sup_{j\geq 1}j^{-1}\big(\mathbb{E}|X|^j\big)^{1/j}.$$
The following version of Bernstein's inequality is due to \cite{Vershynin10}.
\begin{lemma}[Bernstein's Inequality] \label{lem:bernstein}
Let $X_1,...,X_n$ be independent centered random variables with $\max_{1\leq i\leq n}\Norm{X_i}_{\psi_1}\leq K$. Then, there exists $C>0$ such that
$$\mathbb{P}\left(\left|\sum_{i=1}^na_iX_i\right|>t\right)\leq 2\exp\Bigg(-C\min\left\{\frac{t^2}{K^2\Norm{a}^2},\frac{t}{K\Norm{a}_{\infty}}\right\}\Bigg),$$
for any $(a_1,...,a_n)\in\mathbb{R}^{n}$ and $t>0$.
\end{lemma}

\begin{proposition} \label{prop:log}
For any $x,y>0$, we have $|\log x-\log y|\leq \frac{|x-y|}{\min (x,y)}$.
\end{proposition}

We also define $T_{ij}$, which will be very useful in the proof. Observe that under the model assumption, the observation can be represented as
\begin{equation}
X_{ij} = y_j^*T_{ij} + (1-y_j^*)(1-T_{ij}),\quad \text{for all }i\in[n], j\in[m], \label{eq:representation}
\end{equation}
where $T_{ij}$ is a Bernoulli random variable with parameter $p_i^*$ and is independent across $i$ and $j$. Notice $T_{ij}$ has the meaning that the $i$-the worker correctly labels the $j$-th item.

\subsection{Result for General Initialization} \label{sec:genpEM}

The result of Theorem \ref{thm:pEM} is a special case of the following theorem, which uses a general initializer.
\begin{thm} \label{thm:pEMgen}
Let $y^{(t)}$ be the sequence of the projected EM with tuning parameter $\lambda$ satisfies (\ref{eq:asstuning}). Assume for sufficiently large $n$ and $m$, $n\leq m\leq e^n$, $\bar{\nu}\geq n^{-1}\log m$, and the initializer $y^{(0)}$ satisfies
\begin{equation}
\frac{1}{m}\sum_{j\in[m]}|y^{(0)}_j-y_j^*|\leq \sqrt{\frac{\log m}{m}}, \label{eq:initialbound}
\end{equation}
then for any $t\geq 1$ and any $y^*\in\{0,1\}^m$, we have
$$\frac{1}{m}\sum_{j\in[m]}|y_j^{(t)}-y_j^*|\leq \exp\Bigg(-\frac{1}{2}n\max\Big(\bar{\nu}, D(\bar{\mu}_{\lambda}||1-\bar{\mu}_{\lambda})\Big)\Bigg),$$
uniformly over all $t\geq 1$ with probability at least $1-C'/m$ for some positive constant $C'$.
\end{thm}
\begin{proof}
Before stating the main body of the proof. Let us introduce some notation.
Define the  projected version of $p_i^*$ as
\begin{equation}
p_{\lambda,i}^*=\lambda\mathbb{I}\{p_{i}^*<\lambda\}+p_i^*\mathbb{I}\{p_{i}^*\in[\lambda,1-\lambda]\}+(1-\lambda)\mathbb{I}\{p_{i}^*>1-\lambda\}. \label{eq:pstarlambda}
\end{equation}
At the $t$-th iteration, the labeling error of $y^{(t)}$ is denoted as $r^{(t)}=\frac{1}{m}\sum_{j\in[m]}|y_j^{(t)}-y_j^*|$. For $t=0$, $r^{(0)}$ is the labeling error of $y^{(0)}$ and we have $r^{(0)}\leq \sqrt{\frac{\log m}{m}}$ by assumption. Define the events
$$E_1=\left\{\max_{i\in[n]}\left|\frac{1}{m}\sum_{j\in[m]}(T_{ij}-p_i^*)\right|\leq\sqrt{\frac{\log m}{m}}\right\},$$
$$E_2=\left\{\max_{j\in[m]}\left|\sum_{i\in[n]}(T_{ij}-p_i^*)\log\frac{p_{\lambda,i}^*}{1-p_{\lambda,i}^*}\right|\leq 2\log(1/\lambda)\sqrt{n\log m}\right\}.$$
By union bound and Lemma \ref{lem:hoeffding}, we have $\mathbb{P}(E_1^c)\leq C_1/m$ and $\mathbb{P}(E_2^c)\leq C_2/m$, where the second inequality uses the bound $\left|\log\frac{p_{\lambda,i}^*}{1-p_{\lambda,i}^*}\right|\leq 2\log(1/\lambda)$ by the definition of $p_{\lambda,i}^*$. Notice that $\mathbb{P}(E_1\cap E_2)\geq 1-(C_1+C_2)/m$. From now on, our analysis is under the event $E_1\cap E_2$ and is deterministic.

We need the following proposition. Its proof will be stated right after the proof of Theorem \ref{thm:pEMgen}.
\begin{proposition} \label{prop:difflogp}
Under the event $E_1$, as long as $2\lambda+r^{(t-1)}\leq \frac{1}{4}$ and $m\geq 9$, we have
$$\max_{i\in[n]}\left|\log\frac{p_i^{(t)}}{1-p_i^{(t)}}-\log\frac{p_{\lambda,i}^*}{1-p_{\lambda,i}^*}\right|\leq 2\lambda^{-1}\sqrt{\frac{\log m}{m}}+2\lambda^{-1}r^{(t-1)},$$
for all $t\geq 1$.
\end{proposition}

Our first goal is to show that once $r^{(0)}\leq \sqrt{\frac{\log m}{m}}$, then we have $r^{(t)}\leq \sqrt{\frac{\log m}{m}}$ for all $t\geq 1$. By mathematical induction, we assume that  $r^{(t-1)}\leq \sqrt{\frac{\log m}{m}}$ is true. Notice the assumption $2\lambda+r^{(t-1)}\leq 1/4$ in Proposition \ref{prop:difflogp} is satisfied by the range of $\lambda$ in (\ref{eq:asstuning}). Then by the definition of the E-step and the representation (\ref{eq:representation}), we have
\begin{eqnarray}
\label{eq:sudden-s} r^{(t)} &=& \frac{1}{m}\sum_{j}\frac{1}{1+\exp\Big(\sum_i(2T_{ij}-1)\log\frac{p_i^{(t)}}{1-p_i^{(t)}}\Big)} \\
\nonumber &\leq& \frac{1}{m}\sum_j\exp\Bigg(-\sum_i(2T_{ij}-1)\log\frac{p_i^{(t)}}{1-p_i^{(t)}}\Bigg)
\end{eqnarray}
\begin{eqnarray}
\label{eq:useE1} &\leq& \frac{1}{m}\sum_j\exp\Bigg(-\sum_i(2T_{ij}-1)\log\frac{p_{\lambda,i}^*}{1-p_{\lambda,i}^*}+4n\lambda^{-1}\sqrt{\frac{\log m}{m}}\Bigg) \\
\label{eq:useE2} &\leq& \frac{1}{m}\sum_j\exp\Bigg(-\sum_i(2p_i^*-1)\log\frac{p_{\lambda,i}^*}{1-p_{\lambda,i}^*}\Bigg)\\
\label{eq:sudden-e} && \times\exp\Bigg(4n\lambda^{-1}\sqrt{\frac{\log m}{m}}+4\log\Big(\lambda^{-1}\Big)\sqrt{n\log m}\Bigg).
\end{eqnarray}
The inequality (\ref{eq:useE1}) is because of Proposition \ref{prop:difflogp} and the assumption $r^{(t-1)}\leq\sqrt{\frac{\log m}{m}}$. The inequality (\ref{eq:useE2}) is because of $E_2$. Now it remains to bound the two terms in (\ref{eq:useE2}). For an set $A$, denote its cardinality by $|A|$. For the exponent in the first term of (\ref{eq:useE2}), we have
\begin{eqnarray}
\label{eq:s-sudden}&& \sum_i(2p_i^*-1)\log\frac{p_{\lambda,i}^*}{1-p_{\lambda,i}^*} \\
\nonumber&=& \Bigg(\sum_{i:p_i^*<\lambda}+\sum_{i:p_i^*\in[\lambda,1-\lambda]}+\sum_{i:p_i^*>1-\lambda}\Bigg)(2p_i^*-1)\log\frac{p_{\lambda,i}^*}{1-p_{\lambda,i}^*} \\
\nonumber&\geq& \Big(|\{i:p_i^*<\lambda\}|+|\{i:p_i^*>1-\lambda\}|\Big)(1-2\lambda)\log\frac{1-\lambda}{\lambda} \\
\nonumber&& + \sum_{i:p_i^*\in[\lambda,1-\lambda]}D(p_i^*||1-p_i^*) \\
\nonumber&\geq&  \Big(|\{i:p_i^*<\lambda\}|+|\{i:p_i^*>1-\lambda\}|\Big)+\sum_{i:p_i^*\in[\lambda,1-\lambda]}(2p_i^*-1)^2 \\
\label{eq:e-sudden}&\geq& \sum_i(2p^*-1)^2=n\bar{\nu}.
\end{eqnarray}
Another way of bounding it gives
\begin{eqnarray*}
&& \sum_i(2p_i^*-1)\log\frac{p_{\lambda,i}^*}{1-p_{\lambda,i}^*} \geq  \sum_i(2p_{\lambda,i}^*-1)\log\frac{p_{\lambda,i}^*}{1-p_{\lambda,i}^*} \\
&=& \sum_i D(p_{\lambda,i}^*||1-p_{\lambda,i}^*) \geq nD(\bar{\mu}_{\lambda}||1-\bar{\mu}_{\lambda}).
\end{eqnarray*}
Hence, we have
$$\exp\Bigg(-\sum_i(2p_i^*-1)\log\frac{p_{\lambda,i}^*}{1-p_{\lambda,i}^*}\Bigg)\leq \exp\Bigg(-n\max\Big(\bar{\nu},D(\bar{\mu}_{\lambda}||1-\bar{\mu}_{\lambda})\Big)\Bigg).$$
For the exponent in the second term of (\ref{eq:useE2}), we use the range of $\lambda$ in (\ref{eq:asstuning}) to get
\begin{eqnarray*}
&& 4n\Bigg(\lambda^{-1}\sqrt{\frac{\log m}{m}}+\log(\lambda^{-1})\sqrt{\frac{\log m}{n}}\Bigg)  \leq \frac{1}{2}n\bar{\nu}.
\end{eqnarray*}
Therefore,
\begin{equation}
r^{(t)}\leq\exp\Bigg(-\frac{1}{2}n\max\Big(\bar{\nu},D(\bar{\mu}_{\lambda}||1-\bar{\mu}_{\lambda})\Big)\Bigg).\label{eq:desiredforpEM}
\end{equation}
Once $\bar{v}\geq\frac{\log m}{n}$ is satisfied, we have $r^{(t)}\leq e^{-n\bar{\nu}/2}\leq \sqrt{\frac{\log m}{m}}$.
This guarantees that $r^{(t)}\leq \sqrt{\frac{\log m}{m}}$ for every $t$.

Finally, observe that the above argument also implies that as long as $r^{(t-1)}\leq \sqrt{\frac{\log m}{m}}$, we must have (\ref{eq:desiredforpEM}). Thus, the proof is complete.
\end{proof}

\begin{proof}[Proof of Proposition \ref{prop:difflogp}]
Let the M-step of the classical EM be denoted as
\begin{equation}
\check{p}_i^{(t)}=\frac{1}{m}\sum_{j\in[m]}\Big((1-X_{ij})(1-y_j^{(t-1)})+X_{ij}y_j^{(t-1)}\Big). \label{eq:checkp}
\end{equation}
Then, the M-step of the projected EM can be written as
$$p_i^{(t)}=\lambda\mathbb{I}\{\check{p}_i^{(t)}<\lambda\}+\check{p}_i^{(t)}\mathbb{I}\{\check{p}_i^{(t)}\in[\lambda,1-\lambda]\}+(1-\lambda)\mathbb{I}\{\check{p}_i^{(t)}>1-\lambda\}.$$
The definition (\ref{eq:checkp}) and the representation (\ref{eq:representation}) implies that
\begin{equation}
|\check{p}_i^{(t)}-p_i^*|\leq \left|\frac{1}{m}\sum_j(T_{ij}-p_i^*)\right|+r^{(t-1)} \label{eq:boundcheckp}
\end{equation}
For each $i\in[n]$, we have by Proposition \ref{prop:log},
\begin{eqnarray*}
&& \left|\log\frac{p_i^{(t)}}{1-p_i^{(t)}}-\log\frac{p_{\lambda,i}^*}{1-p_{\lambda,i}^*}\right| \leq 2\lambda^{-1}|p_i^{(t)}-p_{\lambda,i}^*| \\
&\leq& 2\lambda^{-1}|\check{p}_i^{(t)}-p_i^*| + 4\lambda^{-1}\mathbb{I}\{|\check{p}_i^{(t)}-p_i^*|>1-2\lambda\} \\
&\leq& 2\lambda^{-1}\left|\frac{1}{m}\sum_j(T_{ij}-p_i^*)\right|+2\lambda^{-1}r^{(t-1)} + 4\lambda^{-1}\mathbb{I}\left\{\left|\frac{1}{m}\sum_j(T_{ij}-p_i^*)\right|>\frac{1}{2}\right\},
\end{eqnarray*}
where the second  inequality is due to the definitions (\ref{eq:checkp}) and (\ref{eq:pstarlambda}), and the last inequality is due to (\ref{eq:boundcheckp}) and the assumption $2\lambda+r^{(t-1)}\leq 1/4$. The event $E_1$ implies that  the first term is bounded by $2\lambda^{-1}\sqrt{\frac{\log m}{m}}$ and the third term above is $0$ as long as $m^{-1}\log m\leq 1/4$. Thus, the proof is complete.
\end{proof}

\subsection{Convergence of Initialization} \label{sec:inipEM}

In this section, we show that the initialization step provides a consistent estimator of the labels. The key of the initialization step is the equation of $\pi$ stated in Proposition \ref{prop:equationforpi}. We first give a proof of this result.
\begin{proof}[Proof of Proposition \ref{prop:equationforpi}]
Define the cross-moment $M_{ik}=\frac{1}{m}\sum_{j\in[m]}\mathbb{E}(X_{ij}X_{kj})$. Together with the definition of $M_i$ in (\ref{eq:ini-M_i-eq}), we have
\begin{eqnarray}
\label{eq:seet1}M_i &=& \pi p_i^*+(1-\pi)(1-p_i^*),\\
\label{eq:seet2}M_k &=& \pi p_k^*+(1-\pi)(1-p_k^*), \\
\label{eq:seet3}M_{ik} &=& \pi p_i^*p_k^*+(1-\pi)(1-p_i^*)(1-p_k^*).
\end{eqnarray}
Plugging (\ref{eq:seet1}) and (\ref{eq:seet2}) into (\ref{eq:seet3}), we have
\begin{equation}
(2M_i+2M_k-1-4M_{ik})\pi^2-(2M_i+2M_k-1-4M_{ik})\pi+M_iM_k-M_{ik}=0.\label{eq:tbaveraged}
\end{equation}
Define
$$M^{(1)}=\frac{1}{nm}\sum_{i\in[n]}\sum_{j\in[m]}\mathbb{E}X_{ij}=\frac{1}{n}\sum_{i\in[n]}M_i,$$
and
$$M^{(2)}=\frac{1}{n^2m}\sum_{i\in[n]}\sum_{k\in[n]}\sum_{j\in[m]}\mathbb{E}(X_{ij}X_{kj})=\frac{1}{n^2}\sum_{i\in[n]}\sum_{k\in[n]}M_{ik}.$$
Averaging the equation (\ref{eq:tbaveraged}) over $i,k\in[n]$, we obtain
\begin{equation}
\pi^2-\pi+\frac{M^{(2)}-[M^{(1)}]^2}{1-4M^{(1)}+4M^{(2)}}=0.\label{eq:unclear}
\end{equation}
Since
\begin{eqnarray*}
M^{(2)}-[M^{(1)}]^2 &=& \frac{1}{m}\sum_j\mathbb{E}Q_j^2 - \left(\frac{1}{m}\sum_j\mathbb{E}Q_j\right)^2 \\
&=&  \frac{1}{m}\sum_j\mathbb{E}Q_j^2 - \frac{1}{m^2}\sum_{jk}\mathbb{E}Q_j\mathbb{E}Q_k \\
&=& \frac{1}{2m^2}\sum_{jk}\left(\mathbb{E}Q_j^2+\mathbb{E}Q_k^2-2\mathbb{E}(Q_jQ_k)\right) \\
&=& \frac{1}{2m^2}\sum_{jk}\mathbb{E}(Q_j-Q_k)^2,
\end{eqnarray*}
and
\begin{eqnarray*}
1-4M^{(1)}+4M^{(2)} &=& \frac{1}{m}\sum_j\left(1-4\mathbb{E}Q_j+4\mathbb{E}Q_j^2\right) \\
&=& \frac{4}{m}\sum_j\mathbb{E}(Q_j-1/2)^2,
\end{eqnarray*}
the equation (\ref{eq:unclear}) is equivalent to (\ref{eq:clear}), and the proof is complete.
\end{proof}

Now let us state the main result of the initialization step.
\begin{thm}\label{thm:revision2}
Set $\bar{\lambda}=1/6$.
Assume $n^2\log m\leq m\leq e^n$, $\frac{\log m}{n}\leq\bar{\nu}\leq 1-n\sqrt{\frac{\log m}{m}}$ and $|2\pi-1|>c$ for some constant $c$. Then there exists some constant $C'>0$, such that
$$\min\left\{\frac{1}{m}\sum_{j\in[m]}|y^{(0)}_j-y_j^*|, \frac{1}{m}\sum_{j\in[m]}|y^{(0)}_j-(1-y_j^*)|\right\}\leq\sqrt{\frac{\log m}{m}},$$
with probability at least $1-C'/m$ for sufficiently large $n$ and $m$.
\end{thm}
To facilitate the proof of Theorem \ref{thm:revision2}, we state three auxiliary lemmas. The proofs of these lemmas will be given in the end of this section. Let us introduce some notation.
$$N=\frac{1}{2m^2}\sum_{jk}\mathbb{E}(Q_j-Q_k)^2,\quad \hat{N}=\frac{1}{2m^2}\sum_{jk}(Q_j-Q_k)^2,$$
$$D=\frac{4}{m}\sum_{j\in[m]}\mathbb{E}(Q_j-1/2)^2,\quad \hat{D}=\frac{4}{m}\sum_{j\in[m]}(Q_j-1/2)^2.$$

\begin{lemma}\label{lem:ini1}
When $|2\pi-1|\geq c$ for some constant $c>0$, we have
$$c\min\left\{|\hat{\pi}-\pi|, |\hat{\pi}-(1-\pi)|\right\}\leq \left(1-\left|1-\frac{\hat{D}}{D}\right|\right)^{-1}\left(\left|1-\frac{\hat{N}}{N}\right| + \left|1-\frac{\hat{D}}{D}\right|\right).$$
\end{lemma}

\begin{lemma}\label{lem:ini2}
Assume $\bar{\nu}\leq 1-4n^{-1}$, and then
$$|\hat{D}-D|\leq D\sqrt{\frac{22\log m}{m}}+\frac{\log m}{3m}$$
 with probability at least $1-2/m$.
\end{lemma}

\begin{lemma}\label{lem:ini3}
Assume $\bar{\nu}\leq 1-4n^{-1}$, and then
$$|\hat{N}-N|\leq 31N\sqrt{\frac{\log m}{m}}+\frac{10\log m}{m},$$
with probability at least $1-6/m$.
\end{lemma}

\begin{proof}[Proof of Theorem \ref{thm:revision2}]
Let us first derive a bound for $\min\left\{|\hat{\pi}-\pi|, |\hat{\pi}-(1-\pi)|\right\}$. Lemma \ref{lem:ini2} and Lemma \ref{lem:ini3} imply that with probability at least $1-8/m$,
\begin{eqnarray}
\label{eq:D/D}\left|1-\frac{\hat{D}}{D}\right| &\leq& \sqrt{\frac{22\log m}{m}}+\frac{\log m}{3Dm}, \\
\label{eq:N/N}\left|1-\frac{\hat{N}}{N}\right| &\leq& 31\sqrt{\frac{\log m}{m}}+\frac{10\log m}{Nm}.
\end{eqnarray}
Note that
\begin{eqnarray}
\nonumber D &=& \frac{4}{m}\sum_{j\in[m]}\left(\mathbb{E}(Q_j-\mathbb{E}Q_j)^2 + (\mathbb{E}Q_j-1/2)^2\right) \\
\label{eq:vD7}&=& \frac{4}{n^2}\sum_{i\in[n]}p_i^*(1-p_i^*) + 4\left(\frac{1}{n}\sum_{i\in[n]}p_i^*-\frac{1}{2}\right)^2.
\end{eqnarray}
Under the assumption that $\bar{\nu}\leq 1-n\sqrt{\frac{\log m}{m}}$,
\begin{equation}
\frac{1}{n}\sum_{i\in[n]}p_i^*(1-p_i^*)=\frac{1}{4}(1-\bar{\nu})\geq \frac{n}{4}\sqrt{\frac{\log m}{m}}.
\end{equation}
Hence, $D\geq\sqrt{\frac{\log m}{m}}$. Similar argument leads to a lower bound for $N$. That is, $N\geq \frac{1}{8}\sqrt{\frac{\log m}{m}}$. For details, see (\ref{eq:outref}) in the proof of Lemma \ref{lem:ini3}. The upper bounds (\ref{eq:D/D}) and (\ref{eq:N/N}), together with the lower bounds for $D$ and $N$, imply that
$$
\left|1-\frac{\hat{D}}{D}\right| \leq \sqrt{\frac{26\log m}{m}},\quad \left|1-\frac{\hat{N}}{N}\right|\leq 111\sqrt{\frac{\log m}{m}},
$$
with probability at least $1-8/m$. Applying Lemma \ref{lem:ini1} and under the assumption that $|2\pi-1|\geq c$, there exists some constant $C$, such that
$$\min\left\{|\hat{\pi}-\pi|, |\hat{\pi}-(1-\pi)|\right\}\leq C\sqrt{\frac{\log m}{m}},$$
with probability at least $1-8/m$. Next, we provide a bound for each $\hat{p}_i$. Since
$$\hat{p}_i=\frac{\hat{M}_i-(1-\hat{\pi})}{2\hat{\pi}-1},\quad p_i^*=\frac{M_i-(1-\pi)}{2\pi-1},$$
we have
\begin{eqnarray*}
|\hat{p}_i-p_i^*| &\leq& \left(1-\frac{|\hat{\pi}-\pi|}{|2\pi-1|}\right)^{-1}\left(\frac{|\hat{\pi}-\pi|}{|2\pi-1|}+\frac{|\hat{M}_i-M_i|}{|2\pi-1|}\right),\\
|\hat{p}_i-(1-p_i^*)| &\leq& \left(1-\frac{|\hat{\pi}-(1-\pi)|}{|2\pi-1|}\right)^{-1}\left(\frac{|\hat{\pi}-(1-\pi)|}{|2\pi-1|}+\frac{|\hat{M}_i-M_i|}{|2\pi-1|}\right).
\end{eqnarray*}
Using the bound for $\min\left\{|\hat{\pi}-\pi|, |\hat{\pi}-(1-\pi)|\right\}$, the assumption $|2\pi-1|>c$ and $\mathbb{P}\left(\max_{i\in[n]}|\hat{M}_i-M_i|>t\right)\leq 2n\exp(-2mt^2)$, we have
$$\min\left\{\max_{i\in[n]}|\hat{p}_i-p_i^*|,\max_{i\in[n]}|\hat{p}_i-(1-p_i^*)|\right\}\leq C_1\sqrt{\frac{\log m}{m}},$$
with probability at least $1-C'/m$. Without loss of generality, we consider the case $\max_{i\in[n]}|\hat{p}_i-p_i^*|\leq C_1\sqrt{\frac{\log m}{m}}$.
Recall the definition of $p_i^{(0)}$, we have
\begin{eqnarray*}
&& \left|\log\frac{p_i^{(0)}}{1-p_i^{(0)}}-\log\frac{p_{\bar{\lambda},i}^*}{1-p_{\bar{\lambda},i}^*}\right| \leq 2\bar{\lambda}^{-1}|p_i^{(0)}-p_{\bar{\lambda},i}^*| \\
&\leq& 2\bar{\lambda}^{-1}|\hat{p}_i-p_i^*| + 4\bar{\lambda}^{-1}\mathbb{I}\{|\hat{p}_i-p_i^*|>1-2\bar{\lambda}\}\leq 2C_2\bar{\lambda}^{-1}\sqrt{\frac{\log m}{m}},
\end{eqnarray*}
uniformly over $i\in[n]$ with probability at least $1-C'/m$. Note that the quantity $\mathbb{I}\{|\hat{p}_i-p_i^*|>1-2\bar{\lambda}\}$ above is zero as long as $1-2\bar{\lambda}>C_2\sqrt{\frac{\log m}{m}}$.

Let us bound the error rate of $y^{(0)}$. Using the same arguments in (\ref{eq:sudden-s})-(\ref{eq:sudden-e}),
\begin{eqnarray}
\label{eq:exponent111}\frac{1}{m}\sum_j|y^{(0)}-y_j^*| &\leq& \frac{1}{m}\sum_j\exp\left(-\sum_i(2p_i^*-1)\log\frac{p_{\bar{\lambda},i}^*}{1-p_{\bar{\lambda},i}^*}\right) \\
\label{eq:exponent222}&&\times \exp\left(2nC_2\bar{\lambda}^{-1}\sqrt{\frac{\log m}{m}}+4n\log(\bar{\lambda}^{-1})\sqrt{\frac{\log m}{n}}\right),
\end{eqnarray}
where the exponent in (\ref{eq:exponent111}) is
$$\sum_i(2p_i^*-1)\log\frac{p_{\bar{\lambda},i}^*}{1-p_{\bar{\lambda},i}^*}\geq n\bar{\nu},$$
whenever $\bar{\lambda}\leq 1/6$ by the arguments in (\ref{eq:s-sudden})-(\ref{eq:e-sudden}), and the exponent in (\ref{eq:exponent222}) is
$$2nC_2\bar{\lambda}^{-1}\sqrt{\frac{\log m}{m}}+4n\log(\bar{\lambda}^{-1})\sqrt{\frac{\log m}{n}}\leq\frac{1}{2}n\bar{\nu},$$
when $\bar{\lambda}\geq \frac{8C_2}{\sqrt{\log m}}\geq 8C_2\bar{\nu}^{-1}\sqrt{\frac{\log m}{m}}$. Therefore, $y^{(0)}$ has error rate bounded by $\exp(-n\bar{\nu}/2)$. When $\bar{\nu}\geq n^{-1}\log m$, this is smaller than $\sqrt{m^{-1}\log m}$. When $\max_{i\in[n]}|\hat{p}_i-(1-p_i^*)|\leq C_1\sqrt{\frac{\log m}{m}}$ holds, the same analysis above also applies to $1-y^{(0)}$, leading to the same error bound for $1-y^{(0)}$. Combining the two cases, we have obtained the desired result for clustering error. Finally, note that the choice $\bar{\lambda}=1/6$ satisfies all the requirements of $\bar{\lambda}$ used in the proof for sufficiently large $n,m$. Thus, the proof is complete.
\end{proof}

\begin{proof}[Proof of Lemma \ref{lem:ini1}]
With the new notation, we have
\begin{eqnarray}
\label{eq:piequation}\pi^2-\pi+\frac{N}{D} &=& 0, \\
\label{eq:hpiequation}\hat{\pi}^2 -\hat{\pi} +\frac{\hat{N}}{\hat{D}} &=& 0.
\end{eqnarray}
Subtracting (\ref{eq:piequation}) by (\ref{eq:hpiequation}), we have
$$|(\hat{\pi}-\pi)(\hat{\pi}-(1-\pi))|=\left|\frac{\hat{N}}{\hat{D}}-\frac{N}{D}\right|.$$
The right hand side of the above equality is bounded by
\begin{eqnarray*}
\left|\frac{N}{D}-\frac{\hat{N}}{\hat{D}}\right| &\leq& \left|\frac{N}{D}\right|\left|1-\frac{\hat{N}}{N}\right| + \left|\frac{\hat{N}}{\hat{D}}\right|\left|1-\frac{\hat{D}}{D}\right| \\
&\leq& \left|\frac{N}{D}\right|\left(\left|1-\frac{\hat{N}}{N}\right| + \left|1-\frac{\hat{D}}{D}\right|\right) + \left|\frac{N}{D}-\frac{\hat{N}}{\hat{D}}\right|\left|1-\frac{\hat{D}}{D}\right|.
\end{eqnarray*}
Thus,
$$|(\hat{\pi}-\pi)(\hat{\pi}-(1-\pi))|\leq \left(1-\left|1-\frac{\hat{D}}{D}\right|\right)^{-1}\left|\frac{N}{D}\right|\left(\left|1-\frac{\hat{N}}{N}\right| + \left|1-\frac{\hat{D}}{D}\right|\right).$$
Since $|(\hat{\pi}-\pi)(\hat{\pi}-(1-\pi))|\geq \frac{c}{2}\min\left\{|\hat{\pi}-\pi|, |\hat{\pi}-(1-\pi)|\right\}$ and $|N/D|=\pi(1-\pi)\leq 1/2$, the proof is complete.
\end{proof}

\begin{proof}[Proof of Lemma \ref{lem:ini2}]
Note that
$$\hat{D}-D=\frac{4}{m}\sum_{j\in[m]}\left((Q_j-1/2)^2-\mathbb{E}(Q_j-1/2)^2\right).$$
To apply Bernstein's inequality, we need to bound the variance of $(Q_j-1/2)^2$. We claim that
\begin{equation}
\sqrt{\frac{1}{m}\sum_{j\in[m]}\text{Var}\left((Q_j-1/2)^2\right)}\leq \sqrt{11}D. \label{eq:varianceD}
\end{equation}
The bound (\ref{eq:varianceD}) will be established in the end of the proof. With (\ref{eq:varianceD}), applying Bernstein's inequality \citep{bernstein27}, we have
$$\mathbb{P}\left(\left|\frac{1}{m}\sum_{j\in[m]}[(Q_j-1/2)^2-\mathbb{E}(Q_j-1/2)^2]\right|>t\right)\leq 2\exp\left(-\frac{mt^2/2}{11D^2+t/6}\right).$$
This completes the proof by choosing an appropriate $t$.

Now let us establish (\ref{eq:varianceD}). Direct calculation gives
\begin{eqnarray}
\nonumber\text{Var}\left((Q_j-1/2)^2\right) &=& \mathbb{E}(Q_j-1/2)^4-\left(\mathbb{E}(Q_j-1/2)^2\right)^2 \\
\nonumber&=& \mathbb{E}(Q_j-\mathbb{E}Q_j+\mathbb{E}Q_j-1/2)^4-\left(\mathbb{E}(Q_j-\mathbb{E}Q_j+\mathbb{E}Q_j-1/2)^2\right)^2 \\
\nonumber&=& \mathbb{E}(Q_j-\mathbb{E}Q_j)^4+4(\mathbb{E}Q_j-1/2)\mathbb{E}(Q_j-\mathbb{E}Q_j)^3 \\
\label{eq:vD1}&& + 4\mathbb{E}(Q_j-\mathbb{E}Q_j)^2(\mathbb{E}Q_j-1/2)^2-\left(\mathbb{E}(Q_j-\mathbb{E}Q_j)^2\right)^2.
\end{eqnarray}
Using H\"{o}lder's inequality, we have $\mathbb{E}(Q_j-\mathbb{E}Q_j)^3\leq \left(\mathbb{E}(Q_j-\mathbb{E}Q_j)^4\right)^{3/4}$, which implies
\begin{equation}
(\mathbb{E}Q_j-1/2)\mathbb{E}(Q_j-\mathbb{E}Q_j)^3\leq (\mathbb{E}Q_j-1/2)^4+ \mathbb{E}(Q_j-\mathbb{E}Q_j)^4. \label{eq:vD2}
\end{equation}
Combining (\ref{eq:vD1}) and (\ref{eq:vD2}) with
$$2\mathbb{E}(Q_j-\mathbb{E}Q_j)^2(\mathbb{E}Q_j-1/2)^2\leq \left(\mathbb{E}(Q_j-\mathbb{E}Q_j)^2\right)^2+(\mathbb{E}Q_j-1/2)^4,$$
we have the bound
\begin{equation}
\text{Var}\left((Q_j-1/2)^2\right)\leq 5\mathbb{E}(Q_j-\mathbb{E}Q_j)^4+\left(\mathbb{E}(Q_j-\mathbb{E}Q_j)^2\right)^2+6(\mathbb{E}Q_j-1/2)^4.\label{eq:vD3}
\end{equation}
The fourth moment in (\ref{eq:vD3}) is bounded by
\begin{eqnarray}
\nonumber\mathbb{E}(Q_j-\mathbb{E}Q_j)^4 &=& \mathbb{E}\left(\frac{1}{n}\sum_{i\in[n]}(X_{ij}-\mathbb{E}X_{ij})\right)^4 \\
\nonumber&=& \frac{1}{n^4}\sum_{i_1i_2i_3i_4}\mathbb{E}(X_{i_1j}-\mathbb{E}X_{i_1j})(X_{i_2j}-\mathbb{E}X_{i_2j})(X_{i_3j}-\mathbb{E}X_{i_3j})(X_{i_4j}-\mathbb{E}X_{i_4j}) \\
\label{eq:vD4}&=&\frac{6}{n^4}\sum_{i\neq k}\mathbb{E}(X_{ij}-\mathbb{E}X_{ij})^2\mathbb{E}(X_{kj}-\mathbb{E}X_{kj})^2+\frac{1}{n^4}\sum_{i\in[n]}\mathbb{E}(X_{ij}-\mathbb{E}X_{ij})^4 \\
\label{eq:vD5}&\leq& \frac{6}{n^4}\sum_{i\neq k}p_i^*(1-p_i^*)p_k^*(1-p_k^*) + \frac{2}{n^4}\sum_{i\in[n]}p_i^*(1-p_i^*) \\
\nonumber&\leq& 6\left(\frac{1}{n^2}\sum_{i\in[n]}p_i^*(1-p_i^*)\right)^2+ \frac{2}{n^4}\sum_{i\in[n]}p_i^*(1-p_i^*).
\end{eqnarray}
The equality (\ref{eq:vD4}) is due to the fact that $\mathbb{E}(X_{ij}-\mathbb{E}X_{ij})=0$ so that the terms in the expansion having factor $\mathbb{E}(X_{ij}-\mathbb{E}X_{ij})$ are all zeros. The inequality (\ref{eq:vD5}) is by $\mathbb{E}(X_{ij}-\mathbb{E}X_{ij})^2=\text{Var}(X_{ij})=p_i^*(1-p_i^*)$ and $\mathbb{E}(X_{ij}-\mathbb{E}X_{ij})^4=p_i^*(1-p_i^*)^3+(1-p_i^*)(p_i^*)^3\leq 2p_i^*(1-p_i^*)$. The second moment in (\ref{eq:vD3}) is
$$\mathbb{E}(Q_j-\mathbb{E}Q_j)^2=\frac{1}{n^2}\sum_{i\in[n]}p_i^*(1-p_i^*).$$
Plugging the bounds for the second and fourth moments into (\ref{eq:vD3}), we get
$$
\text{Var}\left((Q_j-1/2)^2\right)\leq 31\left(\frac{1}{n^2}\sum_{i\in[n]}p_i^*(1-p_i^*)\right)^2+\frac{10}{n^4}\sum_{i\in[n]}p_i^*(1-p_i^*)+6\left(\frac{1}{n}\sum_{i\in[n]}p_i^*-\frac{1}{2}\right)^4.
$$
Under the assumption $\bar{\nu}\leq 1-4n^{-1}$,
\begin{equation}
\frac{1}{n}\sum_{i\in[n]}p_i^*(1-p_i^*)=\frac{1}{4}(1-\bar{\nu})\geq\frac{1}{n}. \label{eq:vDN}
\end{equation}
Then,
\begin{equation}
\text{Var}\left((Q_j-1/2)^2\right)\leq 41\left(\frac{1}{n^2}\sum_{i\in[n]}p_i^*(1-p_i^*)\right)^2+6\left(\frac{1}{n}\sum_{i\in[n]}p_i^*-\frac{1}{2}\right)^4.\label{eq:vD6}
\end{equation}
Combining (\ref{eq:vD6}) and (\ref{eq:vD7}), we have established (\ref{eq:varianceD}).
\end{proof}

\begin{proof}[Proof of Lemma \ref{lem:ini3}]
Let us use the notation $S_0=\{j\in[m]: y_j^*=0\}$ and $S_1=\{j\in[m]: y_j^*=1\}$. Define
\begin{eqnarray*}
\Delta_0 &=& \frac{1}{m_0(m_0-1)}\sum_{\substack{j\neq k\\ j,k\in S_0}}[(Q_j-Q_k)^2-\mathbb{E}(Q_j-Q_k)^2], \\
\Delta_1 &=& \frac{1}{m_1(m_1-1)}\sum_{\substack{j\neq k\\ j,k\in S_1}}[(Q_j-Q_k)^2-\mathbb{E}(Q_j-Q_k)^2], \\
\Delta_2 &=& \frac{1}{m_0m_1}\sum_{(j,k)\in S_0\times S_1}[(Q_j-Q_k)^2-\mathbb{E}(Q_j-Q_k)^2].
\end{eqnarray*}
Then, direct calculation gives
\begin{eqnarray*}
\hat{N}-N &=& \frac{1}{2m^2}\sum_{jk}[(Q_j-Q_k)^2-\mathbb{E}(Q_j-Q_k)^2] \\
&=& \frac{m_0(m_0-1)}{2m^2}\Delta_0 + \frac{m_1(m_1-1)}{2m^2}\Delta_1 + \frac{m_0m_1}{m^2}\Delta_2,
\end{eqnarray*}
where $m_0=|S_0|=(1-\pi)m$ and $m_1=|S_1|=\pi m$. By triangle inequality, we have
\begin{equation}
|\hat{N}-N|\leq \frac{(1-\pi)^2}{2}|\Delta_0| + \frac{\pi^2}{2}|\Delta_1| + \pi(1-\pi)|\Delta_2|.\label{eq:varianceN}
\end{equation}
It is sufficient to upper bound the three terms. To facilitate the proof, we need to introduce two more quantities. Note that for all $j\neq k$ such that $j,k\in S_0$ or $j,k\in S_1$, $\text{Var}[(Q_j-Q_k)^2]$ is a constant, denoted by $\mathcal{V}^2$. For all $(j,k)\in S_0\times S_1$, $\text{Var}[(Q_j-Q_k)^2]$ is also a constant, denoted by $\mathcal{W}^2$. We claim that
\begin{eqnarray}
\label{eq:vN1}\mathcal{V}^2 &\leq& 22\left(\frac{1}{n^2}\sum_{i\in[n]}p_i^*(1-p_i^*)\right)^2, \\
\label{eq:vN2}\mathcal{W}^2 &\leq& 134\left(\frac{1}{n^2}\sum_{i\in[n]}p_i^*(1-p_i^*)\right)^2+6\left(\frac{2}{n}\sum_{i\in[n]}p_i^*-1\right)^4.
\end{eqnarray}
The bounds (\ref{eq:vN1}) and (\ref{eq:vN2}) will be established in the end of the proof.

Let use first bound $|\Delta_0|$ and $|\Delta_1|$ with the help of (\ref{eq:vN1}). Since they have similar forms, we focus on $|\Delta_0|$. We borrow the decoupling trick for U-statistics developed by \cite{hoeffding1963probability}. Without loss of generality, assume $m_0$ is even. The case when $m_0$ is odd can be obtained via slight modification. Define
\begin{equation}
V(x_1,...,x_{m_0})=\frac{2}{m_0}\sum_{j=1}^{m_0/2}(x_{2j-1}-x_{2j})^2. \label{eq:UstatV}
\end{equation}
Therefore,
$$\Delta_0=\frac{1}{m_0!}\sum_{\sigma(S_0)}[V(Q_{\sigma(1)},...,Q_{\sigma(m_0)})-\mathbb{E}V(Q_{\sigma(1)},...,Q_{\sigma(m_0)})],$$
where the summation is over all permutation of the set $S_0$. By Jensen's inequality,
$$\mathbb{E}\exp(\lambda \Delta_0)\leq \frac{1}{m_0!}\sum_{\sigma(S_0)}\mathbb{E}\exp\left(\lambda[V(Q_{\sigma(1)},...,Q_{\sigma(m_0)})-\mathbb{E}V(Q_{\sigma(1)},...,Q_{\sigma(m_0)})]\right),$$
for any $\lambda>0$. Therefore, the Chernoff bound for $V(Q_{\sigma(1)},...,Q_{\sigma(m_0)})-\mathbb{E}V(Q_{\sigma(1)},...,Q_{\sigma(m_0)})$ is also the Chernoff bound for $\Delta_0$. According to the form (\ref{eq:UstatV}), $V(Q_{\sigma(1)},...,Q_{\sigma(m_0)})$ is average of i.i.d. random variables with variance $\mathcal{V}^2$. The standard moment generation bound can be applied (see, for example, \cite{Vershynin10}). Therefore,
$$\mathbb{P}(|\Delta_0|>t)\leq 2\exp\left(-\frac{3m_0t^2}{12\mathcal{V}^2+4t}\right).$$
Choosing an appropriate $t$, then with probability at least $1-2/m$,
\begin{equation}
|\Delta_0|\leq 2\mathcal{V}\sqrt{\frac{2\log m}{3m_0}}+\frac{4\log m}{3m_0}. \label{eq:haowan0}
\end{equation}
Similar argument leads to
\begin{equation}
|\Delta_1|\leq 2\mathcal{V}\sqrt{\frac{2\log m}{3m_1}}+\frac{4\log m}{3m_1}, \label{eq:haowan1}
\end{equation}
with probability at least $1-2/m$.

Let us then derive a bound for $|\Delta_2|$. Without loss of generality, assume $m_1\geq m_0$. That is, $m_1\geq\frac{m}{2}$. In this case, we write $\Delta_2$ as
$$\Delta_2=\frac{1}{m_0}\sum_{j\in S_0}\left(\frac{1}{m_1}\sum_{k\in S_1}[(Q_j-Q_k)^2-\mathbb{E}(Q_j-Q_k)^2]\right).$$
By Jensen's inequality, we have
$$\mathbb{E}\exp(\lambda\Delta_2)\leq \frac{1}{m_0}\sum_{j\in S_0}\mathbb{E}\exp\left(\frac{\lambda}{m_1}\sum_{k\in S_1}[(Q_j-Q_k)^2-\mathbb{E}(Q_j-Q_k)^2]\right),$$
for any $\lambda>0$. Thus, the Chernoff bound for $\frac{1}{m_1}\sum_{k\in S_1}[(Q_j-Q_k)^2-\mathbb{E}(Q_j-Q_k)^2]$ is also the Chernoff bound for $\Delta_2$. Note that the former quantity is average of i.i.d. random variables with variance $\mathcal{W}^2$. The standard moment generation bound can be applied to obtain
$$\mathbb{P}\left(|\Delta_2|>t\right)\leq 2\exp\left(-\frac{3m_1t^2}{6\mathcal{W}^2+2t}\right).$$
Choosing an appropriate $t$, then by $m_1^{-1}\leq 2m^{-1}$, we have with probability at least $1-2/m$,
\begin{equation}
|\Delta_2| \leq 2\mathcal{W}\sqrt{\frac{\log m}{m}} + \frac{4\log m}{3m}. \label{eq:haowan2}
\end{equation}
The same bound can also be obtain when $m_1<m_0$ via symmetry.

Plugging (\ref{eq:haowan0}), (\ref{eq:haowan1}) and (\ref{eq:haowan2}) into (\ref{eq:varianceN}) and using union bound, we have with probability at least $1-6/m$,
\begin{eqnarray*}
|\hat{N}-N| &\leq& (1-\pi)^2\mathcal{V}\sqrt{\frac{2\log m}{3m_0}}+\pi^2\mathcal{V}\sqrt{\frac{2\log m}{3m_1}} + \frac{2(1-\pi)^2\log m}{3m_0} + \frac{2\pi^2\log m}{3m_1} \\
&& + 2\pi(1-\pi)\mathcal{W}\sqrt{\frac{\log m}{m}} + \frac{4\pi(1-\pi)\log m}{m} \\
&\leq& \left((1-\pi)^{3/2}+\pi^{3/2}\right)\mathcal{V}\sqrt{\frac{2\log m}{3m_0}} +2\pi(1-\pi)\mathcal{W}\sqrt{\frac{\log m}{m}}\\
&& + \frac{2\log m}{3m}+\frac{4\pi(1-\pi)\log m}{m} \\
&\leq& \left(\sqrt{2/3}\mathcal{V}+2\pi(1-\pi)\mathcal{W}\right)\sqrt{\frac{\log m}{m}} + \frac{10\log m}{3m} \\
&\leq& \left(\frac{15.5}{n^2}\sum_{i\in[n]}p_i^*(1-p_i^*)+5\pi(1-\pi)\left(\frac{2}{n}\sum_{i\in[n]}p_i^*-1\right)^2\right)\sqrt{\frac{\log m}{m}} + \frac{10\log m}{3m},
\end{eqnarray*}
where the last inequality uses the bounds (\ref{eq:vN1}) and (\ref{eq:vN2}).
On the other hand,
\begin{eqnarray}
\nonumber N &=& \frac{1}{2m^2}\sum_{jk}\mathbb{E}(Q_j-Q_k)^2 \\
\nonumber &=& \frac{1}{2m^2}\sum_{\substack{j\neq k\\j,k\in S_0}}\mathbb{E}(Q_j-Q_k)^2+ \frac{1}{2m^2}\sum_{\substack{j\neq k\\j,k\in S_1}}\mathbb{E}(Q_j-Q_k)^2 + \frac{1}{m^2}\sum_{(j,k)\in S_0\times S_1}\mathbb{E}(Q_j-Q_k)^2 \\
\nonumber &=& \frac{m(m-1)}{m^2}\frac{1}{n^2}\sum_{i\in[n]}p_i^*(1-p_i^*) + 2\pi(1-\pi)\left(\frac{2}{n}\sum_{i\in[n]}p_i^*-1\right)^2 \\
\label{eq:outref} &\geq& \frac{1}{2n^2}\sum_{i\in[n]}p_i^*(1-p_i^*) + 2\pi(1-\pi)\left(\frac{2}{n}\sum_{i\in[n]}p_i^*-1\right)^2,
\end{eqnarray}
where the last inequality is due to $m\geq 2$.
Combining the upper bound for $|\hat{N}-N|$ and the lower bound for $N$,
$$|\hat{N}-N|\leq 31N\sqrt{\frac{\log m}{m}}+\frac{10\log m}{m},$$
with probability at least $1-6/m$.

Now we establish (\ref{eq:vN1}). For some $j\neq k$ such that $j,k\in S_0$ or $j,k\in S_1$, $\mathcal{V}^2$ can be expressed as
\begin{equation}
\mathcal{V}^2=\mathbb{E}(Q_j-Q_k)^4-\left(\mathbb{E}(Q_j-Q_k)^2\right)^2.\label{eq:vN3}
\end{equation}
The fourth moment is bounded as
\begin{eqnarray}
\nonumber\mathbb{E}(Q_j-Q_k)^4 &=& \mathbb{E}\left(\frac{1}{n}\sum_{i\in[n]}(X_{ij}-X_{ik})\right)^4 \\
\nonumber&=& \frac{1}{n^4}\sum_{i_1i_2i_3i_4}\mathbb{E}(X_{i_1j}-X_{i_1k})(X_{i_2j}-X_{i_2k})(X_{i_3j}-X_{i_3k})(X_{i_4j}-X_{i_4k}) \\
\label{eq:vN4}&=& \frac{6}{n^4}\sum_{i\neq l}\mathbb{E}(X_{ij}-X_{ik})^2\mathbb{E}(X_{lj}-X_{lk})^2 + \frac{1}{n^4}\sum_{i\in[n]}\mathbb{E}(X_{ij}-X_{ik})^4 \\
\label{eq:vN5}&=& \frac{24}{n^4}\sum_{i\neq l}p_i^*(1-p_i^*)p_l^*(1-p_l^*) + \frac{2}{n^4}\sum_{i\in[n]}p_i^*(1-p_i^*) \\
\nonumber&\leq& 24\left(\frac{1}{n^2}\sum_{i\in[n]}p_i^*(1-p_i^*)\right)^2 +  \frac{2}{n^4}\sum_{i\in[n]}p_i^*(1-p_i^*).
\end{eqnarray}
The equality (\ref{eq:vN4}) is because $\mathbb{E}(X_{ij}-X_{ik})=0$ so that the terms in the expansion having factor $\mathbb{E}(X_{ij}-X_{ik})$ are all zeros. The equality (\ref{eq:vN5}) is because $\mathbb{E}(X_{ij}-X_{ik})^2=\text{Var}(X_{ij}-X_{ik})=\text{Var}(X_{ij})+\text{Var}(X_{ik})=2p_i^*(1-p_i^*)$ and  $\mathbb{E}(X_{ij}-X_{ik})^4=2p_i^*(1-p_i^*)$. The second moment is
$$\mathbb{E}(Q_j-Q_k)^2 = \frac{2}{n^2}\sum_{i\in[n]}p_i^*(1-p_i^*).$$
Plugging the bounds for the second and the fourth moments into (\ref{eq:vN3}), we obtain
$$\mathcal{V}^2 \leq 20\left(\frac{1}{n^2}\sum_{i\in[n]}p_i^*(1-p_i^*)\right)^2 +  \frac{2}{n^4}\sum_{i\in[n]}p_i^*(1-p_i^*).$$
The assumption $\bar{\nu}\leq 1-4n^{-1}$ and the argument (\ref{eq:vDN}) leads to the bound (\ref{eq:vN1}).

Finally, we establish (\ref{eq:vN2}). By the same argument that we used to derive (\ref{eq:vD3}), we have
\begin{eqnarray}
\nonumber\mathcal{W}^2 &\leq& 5\mathbb{E}[Q_j-Q_k-\mathbb{E}(Q_j-Q_k)]^4 \\
\label{eq:vN0}&& + \left(\mathbb{E}[Q_j-Q_k-\mathbb{E}(Q_j-Q_k)]^2\right)^2 + 6\left(\mathbb{E}(Q_j-Q_k)\right)^4,
\end{eqnarray}
for some $(j,k)\in S_0\times S_1$. Again, we bound the fourth moment by
\begin{eqnarray}
\nonumber && \mathbb{E}[Q_j-Q_k-\mathbb{E}(Q_j-Q_k)]^4 \\
\nonumber  &=& \mathbb{E}\left(\frac{1}{n}\sum_{i\in[n]}[(X_{ij}-\mathbb{E}X_{ij})-(X_{ik}-\mathbb{E}X_{ik})]\right)^4 \\
\nonumber  &=& \frac{1}{n^4}\sum_{i_1i_2i_3i_4}\mathbb{E}\Big([(X_{i_1j}-\mathbb{E}X_{i_1j})-(X_{i_1k}-\mathbb{E}X_{i_1k})][(X_{i_2j}-\mathbb{E}X_{i_2j})-(X_{i_2k}-\mathbb{E}X_{i_2k})] \\
\nonumber && [(X_{i_3j}-\mathbb{E}X_{i_3j})-(X_{i_3k}-\mathbb{E}X_{i_3k})][(X_{i_4j}-\mathbb{E}X_{i_4j})-(X_{i_4k}-\mathbb{E}X_{i_4k})]\Big) \\
\label{eq:vN6} &=& \frac{6}{n^4}\sum_{i\neq l}\mathbb{E}[(X_{ij}-\mathbb{E}X_{ij})-(X_{ik}-\mathbb{E}X_{ik})]^2[(X_{lj}-\mathbb{E}X_{lj})-(X_{lk}-\mathbb{E}X_{lk})]^2 \\
\nonumber  && + \frac{1}{n^4}\sum_{i\in[n]}\mathbb{E}[(X_{ij}-\mathbb{E}X_{ij})-(X_{ik}-\mathbb{E}X_{ik})]^4 \\
\label{eq:vN7}&=& \frac{24}{n^4}\sum_{i\neq l}p_i^*(1-p_i^*)p_l^*(1-p_l^*) + \frac{2}{n^4}\sum_{i\in[n]}p_i^*(1-p_i^*) \\
\nonumber&\leq& 24\left(\frac{1}{n^2}\sum_{i\in[n]}p_i^*(1-p_i^*)\right)^2 +  \frac{2}{n^4}\sum_{i\in[n]}p_i^*(1-p_i^*).
\end{eqnarray}
The equality (\ref{eq:vN6}) is derived by the same argument for (\ref{eq:vD4}) and (\ref{eq:vN4}). The equality (\ref{eq:vN7}) is because $\mathbb{E}[(X_{ij}-\mathbb{E}X_{ij})-(X_{ik}-\mathbb{E}X_{ik})]^2=\text{Var}(X_{ij})+\text{Var}(X_{ik})=2p_i^*(1-p_i^*)$ and $\mathbb{E}[(X_{ij}-\mathbb{E}X_{ij})-(X_{ik}-\mathbb{E}X_{ik})]^4=\mathbb{E}(X_{ij}-X_{ik})^4=2p_i^*(1-p_i^*)$. The second moment is
$$\mathbb{E}[Q_j-Q_k-\mathbb{E}(Q_j-Q_k)]^2=\text{Var}(Q_j)+\text{Var}(Q_k)=\frac{2}{n^2}\sum_{i\in[n]}p_i^*(1-p_i^*).$$
Plugging the bounds for the second and the fourth moments into (\ref{eq:vN0}), we obtain
$$\mathcal{W}^2 \leq 124\left(\frac{1}{n^2}\sum_{i\in[n]}p_i^*(1-p_i^*)\right)^2 +  \frac{10}{n^4}\sum_{i\in[n]}p_i^*(1-p_i^*)+6\left(\frac{2}{n}\sum_{i\in[n]}p_i^*-1\right)^4.$$
The assumption $\bar{\nu}\leq 1-4n^{-1}$ and the argument (\ref{eq:vDN}) leads to the bound (\ref{eq:vN2}).
\end{proof}

\subsection{Proofs of Theorem \ref{thm:pEM} and Theorem \ref{thm:EM}} \label{sec:finalpEM}

\begin{proof}[Proof of Theorem \ref{thm:pEM}]
A union bound argument implies that the results of Theorem \ref{thm:pEMgen} and Theorem \ref{thm:revision2} hold simultaneously with probability at least $1-C'/m$ for some constant $C'>0$.
Theorem \ref{thm:revision2} implies either $y^{(0)}$ or $1-y^{(0)}$ satisfies the bound (\ref{eq:initialbound}) in Theorem \ref{thm:pEMgen}. By symmetry of the projected EM algorithm, we obtain the clustering error rate. Now we derive the labeling error rate. For simplicity of notation, we use $\hat{y},\hat{p},\hat{r}$ to denote $y^{(t)}, \check{p}, r^{(t)}$. 
Then
\begin{equation}
\tilde{y}=\hat{y}\mathbb{I}\left\{\frac{1}{n}\sum_i\hat{p}_i\geq\frac{1}{2}\right\}+(1-\hat{y})\mathbb{I}\left\{\frac{1}{n}\sum_i\hat{p}_i<\frac{1}{2}\right\}. \label{eq:charconstopt}
\end{equation}
We are going to show $\frac{1}{m}\sum_j|\tilde{y}_j-y_j^*|=\min(\hat{r},1-\hat{r})$.
The characterization of $\tilde{y}$ by (\ref{eq:charconstopt}) implies that
\begin{eqnarray*}
&& \Bigg(\left\{\hat{r}\leq 1-\hat{r}, \frac{1}{n}\sum_i\hat{p}_i<\frac{1}{2}\right\}\bigcup\left\{\hat{r}>1-\hat{r},\frac{1}{n}\sum_i(1-\hat{p}_i)<\frac{1}{2}\right\}\Bigg)^c \\
&\subset& \left\{\frac{1}{m}\sum_j|\tilde{y}_j-y_j^*|=\min(\hat{r},1-\hat{r})\right\}.
\end{eqnarray*}
Hence, it is sufficient to upper bound
\begin{equation}
\mathbb{P}\left\{\hat{r}\leq 1-\hat{r}, \frac{1}{n}\sum_i\hat{p}_i<\frac{1}{2}\right\}+\mathbb{P}\left\{\hat{r}>1-\hat{r},\frac{1}{n}\sum_i(1-\hat{p}_i)<\frac{1}{2}\right\}. \label{eq:pfthm0}
\end{equation}
Without loss of generality, we only bound the first term of (\ref{eq:pfthm0}), because the second term can be bounded in the same way. By direct calculation using the representation (\ref{eq:representation}), this leads to
\begin{equation}
\hat{p}_i = \frac{1}{m}\sum_jT_{ij}+\frac{1}{m}\sum_j(1-2T_{ij})|\hat{y}_j-y_j^*|. \label{eq:handT}
\end{equation}
Summing (\ref{eq:handT}) over $i$ gives the bound
\begin{equation}
\left|\frac{1}{n}\sum_i\hat{p}_i-\frac{1}{nm}\sum_i\sum_jT_{ij}\right| \leq \hat{r}. \label{eq:sumhandT}
\end{equation}
Using the bound (\ref{eq:sumhandT}) and $\mathbb{P}\left\{\min(\hat{r},1-\hat{r})>e^{-n\bar{\nu}/2}\right\}\leq C'/m$ from the clustering rate, we have
\begin{eqnarray*}
&& \mathbb{P}\left\{\hat{r}\leq 1-\hat{r}, \frac{1}{n}\sum_i\hat{p}_i<\frac{1}{2}\right\} \\
&\leq& \mathbb{P}\left\{\hat{r}\leq 1-\hat{r}, \frac{1}{n}\sum_i\hat{p}_i<\frac{1}{2}, \min(\hat{r},1-\hat{r})\leq e^{-n\bar{\nu}/2}\right\} + \mathbb{P}\left\{\min(\hat{r},1-\hat{r})>e^{-n\bar{\nu}/2}\right\} \\
&\leq& \mathbb{P}\left\{\frac{1}{n}\sum_i\hat{p}_i<\frac{1}{2}, \hat{r}\leq e^{-n\bar{\nu}/2}\right\} + C'/m \\
&\leq& \mathbb{P}\left\{\frac{1}{nm}\sum_i\sum_j T_{ij}\leq e^{-n\bar{\nu}/2}+\frac{1}{2} \right\}+C'/m.
\end{eqnarray*}
Under the assumption (\ref{eq:asspave}) and the condition on $\bar{\nu}$, we can use Hoeffding's inequality in Lemma \ref{lem:hoeffding} to get
\begin{eqnarray*}
&&  \mathbb{P}\left\{\frac{1}{nm}\sum_i\sum_j T_{ij}\leq e^{-n\bar{\nu}/8}+\frac{1}{2} \right\} \\
&\leq&  \mathbb{P}\left\{\frac{1}{nm}\sum_i\sum_j (T_{ij}-p_i^*)\leq -\sqrt{\frac{\log m}{nm}}\right\} \\
&\leq& C_1m^{-2}.
\end{eqnarray*}
Therefore, (\ref{eq:pfthm0}) is bounded by $C_2/m$ for some constant $C_2>0$ and $\frac{1}{m}\sum_j|\tilde{y}_j-y_j^*|=\min(\hat{r},1-\hat{r})$ with probability at least $1-C_2/m$. By the conclusions of the clustering error rate, the proof is complete.
\end{proof}

\begin{proof}[Proof of Theorem \ref{thm:EM}]
Compare EM with projected EM, it is sufficient to prove $\check{p}_i^{(t)}=p_i^{(t)}$ for every $i\in[n]$ and $t\geq 1$.
Define the event $E_0=\left\{r^{(0)}\leq \sqrt{\frac{\log m}{m}}\right\}$. Then the event $E_0\cap E_1\cap E_2$ occurs with probability at least $1-C/m$. The following analysis assumes $E_0\cap E_1\cap E_2$.
Remember the definition of $\check{p}_i^{(t)}$ in (\ref{eq:checkp}). By (\ref{eq:boundcheckp}), we have $|\check{p}_i^{(t)}-p_i^*|\leq 2\sqrt{\frac{\log m}{m}}$ for all $i\in[n]$ and $t\geq 1$. Under the assumption (\ref{eq:boundedability}), we have $\check{p}_i^{(t)}\in[\lambda,1-\lambda]$ and this implies $\check{p}_i^{(t)}=p_i^{(t)}$ for all $i\in[n]$ and $t\geq 1$. Thus, the proof is complete.
\end{proof}

\section{Proofs of Theorem \ref{thm:pEMwa}}

In this section, we gather the proofs of the results for estimating workers' abilities. For simplicity of notation, we use $\hat{p}$ and $\hat{r}$ to denote $\tilde{p}$ and $r^{(t)}$. By the definition of $\tilde{p}$ in (\ref{eq:pEMtp}) and the representation (\ref{eq:representation}), we have $\max_{i\in[n]}|\hat{p}_i-\frac{1}{m}\sum_jT_{ij}|\leq \hat{r}$. Under the assumption, we have $\hat{r}\leq \exp\big(-n\bar{\nu}/2\big)\leq m^{-3/4}$ with probability at least $1-C'/m$. Hence,
\begin{equation}
\max_{i\in[n]}\left|\hat{p}_i-\frac{1}{m}\sum_{j}T_{ij}\right| \leq m^{-3/4}, \label{eq:phatandT}
\end{equation}
with probability at least $1-C'/m$. We prove the non-asymptotic bounds using this result.

\begin{proof}[Proofs of (\ref{eq:linfty}) and (\ref{eq:l2})]
By (\ref{eq:phatandT}), with probability at least $1-C'/m$, we have
\begin{equation}
||\hat{p}-p^*||_{\infty}\leq \max_{i\in[n]}\left|\frac{1}{m}\sum_j(T_{ij}-p_i^*)\right|+m^{-3/4},\label{eq:inftyinproof}
\end{equation}
By Lemma \ref{lem:hoeffding}, the first term of (\ref{eq:inftyinproof}) can be bounded by $\sqrt{\frac{\log m}{m}}$ with probability at least $1-C_1/m$. Thus, the bound for $||\hat{p}-p^*||_{\infty}$ immediately follows. Using (\ref{eq:phatandT}), we have
$$\frac{1}{n}||\hat{p}-p^*||^2\leq \frac{1}{n}\sum_i\left|\frac{1}{m}\sum_j(T_{ij}-p_i^*)\right|^2+m^{-3/2}.$$
For each $i\in[n]$, by Lemma \ref{lem:hoeffding}, $\left|\frac{1}{\sqrt{m}}\sum_j(T_{ij}-p_i^*)\right|^2$ is an sub-exponential random variable with bounded sub-exponential norm (see Section 5.2 of \cite{Vershynin10}), and
$$\mathbb{E}\left|\frac{1}{m}\sum_j(T_{ij}-p_i^*)\right|^2=\text{Var}\left(\frac{1}{m}\sum_jT_{ij}\right)=\frac{1}{m}p_i^*(1-p_i^*).$$
By Lemma \ref{lem:bernstein}, $\frac{1}{n}\sum_i\left|\frac{1}{m}\sum_j(T_{ij}-p_i^*)\right|^2\leq \frac{1}{nm}\sum_i p_i^*(1-p_i^*)+C\sqrt{\frac{\log m}{m^2n}}$ with probability at least $1-C_2/m$ for some  constants $C$ and $C_2$. Thus, the bound for $\frac{1}{n}||\hat{p}-p^*||^2$ follows.
\end{proof}

Now we prove Theorem (\ref{eq:CLT}) and (\ref{eq:hdCLT}. Since (\ref{eq:CLT}) is a direct application of Slutsky's theorem in view of (\ref{eq:phatandT}), we only state the proof of (\ref{eq:hdCLT}). The result (\ref{eq:hdCLT}) is an application of the recent development of high-dimensional central limit theorem in econometrics due to \cite{chernozhukov12}. We state a special case of their result as the  lemma below.

\begin{lemma} \label{lem:hdCLT}
Let $U_1,...,U_m$ be i.i.d. sub-Gaussian vectors in $\mathbb{R}^n$ with identity covariance $I$. Then, we have
$$\sup_{t\in\mathbb{R}}\left|\mathbb{P}\left(\max_{i\in[n]}\left|\frac{1}{\sqrt{m}}\sum_{j\in[m]}\Big(U_j(i)-\mathbb{E}U_j(i)\Big)\right|\leq t\right)-\mathbb{P}\Big(\max_{i\in[n]}|Z_i|\leq t\Big)\right|\leq \frac{C\log n}{m^{1/8}},$$
where $Z_1,...,Z_n$ are i.i.d. $N(0,1)$ and $C$ is an absolute constant.
\end{lemma}

The next lemma characterizes the perturbation of the distribution function of $\max_{i\in[n]}|Z_i|$. It is due to \cite{wasserman13}

\begin{lemma} \label{lem:perturb}
Consider $Z_1,...,Z_n$ i.i.d. $N(0,1)$. There is some absolute constant $C>0$, such that for every $\epsilon>0$, we have
$$\sup_{t\in\mathbb{R}}\left|\mathbb{P}\Big(\max_{i\in[n]}|Z_i|\leq t+\epsilon\Big)-\mathbb{P}\Big(\max_{i\in[n]}|Z_i|\leq t\Big)\right|\leq C\epsilon\sqrt{\log\Big(\frac{n}{\epsilon}\Big)}.$$
\end{lemma}

\begin{proof}[Proof of (\ref{eq:hdCLT})]
Define the following quantities
$$A=\max_{i\in[n]}\left|\frac{1}{\sqrt{m}}\sum_{j\in[m]}\frac{T_{ij}-p_i^*}{\sqrt{p_i^*(1-p_i^*)}}\right|,\quad B=\max_{i\in[n]}\left|\frac{\sqrt{m}(\hat{p}_i-p_i^*)}{\sqrt{p_i^*(1-p_i^*)}}\right|-A.$$
Using notation $Q(t)=\mathbb{P}\Big(\max_{i\in[n]}Z_i\leq t\Big)$, we have
\begin{eqnarray*}
&& \mathbb{P}\Big(A+B\leq t\Big)-Q(t) \\
&\leq&  \mathbb{P}\Big(A+B\leq t, |B|\leq\epsilon\Big)+\mathbb{P}\Big(|B|>\epsilon\Big)-Q(t) \\
&\leq&  \mathbb{P}\Big(A\leq t+\epsilon \Big)-Q(t+\epsilon) + \mathbb{P}\Big(|B|>\epsilon\Big) + \left|Q(t+\epsilon)-Q(t)\right|.
\end{eqnarray*}
Similarly,
\begin{eqnarray*}
&& Q(t) - \mathbb{P}\Big(A+B\leq t\Big) \\
&\leq& Q(t-\epsilon)-\mathbb{P}\Big(A\leq t-\epsilon\Big) + \mathbb{P}\Big(|B|>\epsilon\Big) + \left|Q(t-\epsilon)-Q(t)\right|.
\end{eqnarray*}
Therefore, $\sup_{t}\left|Q(t) - \mathbb{P}\Big(A+B\leq t\Big)\right|$ is bounded by
$$\sup_{t}\left|Q(t) - \mathbb{P}\Big(A\leq t\Big)\right|+\mathbb{P}\Big(|B|>\epsilon\Big)+\sup_t\left|Q(t+\epsilon)-Q(t)\right|.$$
The first term and the third term above are bounded by Lemma \ref{lem:hdCLT} and Lemma \ref{lem:perturb} respectively. By (\ref{eq:phatandT}) and the assumption (\ref{eq:assboundp}), we have
$$\mathbb{P}\Big(|B|>\sqrt{C'/m}\Big)\leq C_1/m.$$
Hence, letting $\epsilon=\sqrt{C'/m}$, we have
\begin{eqnarray*}
&& \sup_{t}\left|Q(t) - \mathbb{P}\Big(A+B\leq t\Big)\right| \\
&\leq& \frac{C\log n}{m^{1/8}} + C\sqrt{\frac{C'}{m}}\sqrt{\log n+\frac{1}{2}\log\frac{m}{C'}} + C_1/m \\
&\leq& \frac{C_2\log n}{m^{1/8}},
\end{eqnarray*}
for some positive constant $C_2$. Thus, the proof is complete.
\end{proof}

\section{Proofs of Theorem  \ref{thm:lower1} and Theorem  \ref{thm:lower2}} \label{sec:prooflower}

\begin{proof}[Proof of Theorem \ref{thm:lower1}]
Consider the parameter space $\mathcal{P}_{\ceil{n\bar{\nu}}/n}$ defined in (\ref{eq:paraspacenu}). Remember the least favorable case $\mathcal{P}'$ we have constructed in (\ref{eq:lfq}). Then, we have the following reduction
\begin{eqnarray*}
\sup_{y\in\{0,1\}^m,p\in\mathcal{P}_{\bar{\nu}}}E_{p,y}\left(\frac{1}{m}\sum_{j\in[m]}|\hat{y}_j-y_j|\right) &\geq& \sup_{y\in\{0,1\}^m,p\in\mathcal{P}_{\ceil{n\bar{\nu}}/n}}E_{p,y}\left(\frac{1}{m}\sum_{j\in[m]}|\hat{y}_j-y_j|\right) \\
&\geq&  \sup_{y\in\{0,1\}^m,p\in\mathcal{P}'}E_{p,y}\left(\frac{1}{m}\sum_{j\in[m]}|\hat{y}_j-y_j|\right),
\end{eqnarray*}
for any $\hat{y}\in[0,1]^m$. The first inequality is because $\ceil{n\bar{\nu}}/n\geq\bar{\nu}$, and the second inequality is because $\mathcal{P}_{\ceil{n\bar{\nu}}/n}\supset\mathcal{P}'$. Consider uniform prior on both $\{0,1\}^m$ and $\mathcal{P}'$, and we have a further reduction by
\begin{eqnarray}
\nonumber &&  \sup_{y\in\{0,1\}^m,p\in\mathcal{P}'}E_{p,y}\left(\frac{1}{m}\sum_{j\in[m]}|\hat{y}_j-y_j|\right) \\
\nonumber &\geq& \frac{1}{m|\mathcal{P}'|}\sum_{j\in[m]}\sum_{p\in\mathcal{P}'}\Bigg(\frac{1}{2}E_{p,y_j=1}|\hat{y}_j-1|+\frac{1}{2}E_{p,y_j=0}|\hat{y}_j-0|\Bigg) \\
\label{eq:bayesriskinproof} &\geq& \frac{1}{m|\mathcal{P}'|}\sum_{j\in[m]}\sum_{p\in\mathcal{P}'}\Bigg(\frac{1}{2}E_{p,y_j=1}(\hat{y}_j-1)^2+\frac{1}{2}E_{p,y_j=0}(\hat{y}_j-0)^2\Bigg),
\end{eqnarray}
for any $\hat{y}\in[0,1]^m$. We have bounded the maximum risk from below by the Bayes risk (\ref{eq:bayesriskinproof}). Define the set $S_{p'}=\{i\in[n]: p_i'=1\}$ for every $p'\in\mathcal{P}'$. The Bayes risk (\ref{eq:bayesriskinproof}) can be minimized by the Bayes estimator (Chapter 5 of \cite{lehmann98}), which has the form
\begin{eqnarray*}
\hat{y}_j &\propto& \sum_{p'\in\mathcal{P}'}P_{p',y_j=1}(X_{1j},...,X_{nj}) = \sum_{p'\in\mathcal{P}'}\Bigg(\frac{1}{2}\Bigg)^{n-\ceil{n\bar{\nu}}}\prod_{i\in S_{p'}}\mathbb{I}\{X_{ij}=1\}, \\
1-\hat{y}_j &\propto& \sum_{p'\in\mathcal{P}'}P_{p',y_j=0}(X_{1j},...,X_{nj}) = \sum_{p'\in\mathcal{P}'}\Bigg(\frac{1}{2}\Bigg)^{n-\ceil{n\bar{\nu}}}\prod_{i\in S_{p'}}\mathbb{I}\{X_{ij}=0\} ,
\end{eqnarray*}
via Bayes formula. By symmetry of the distribution and the representation (\ref{eq:representation}), for any $p\in\mathcal{P}'$ and $j\in[m]$, each term of (\ref{eq:bayesriskinproof}) can be written as
\begin{eqnarray*}
&&\frac{1}{2}E_{p,y_j=1}(\hat{y}_j-1)^2+\frac{1}{2}E_{p,y_j=0}(\hat{y}_j-0)^2 \\
&=& E_p\left(\frac{\sum_{p'\in\mathcal{P}'}\prod_{i\in S_{p'}}\mathbb{I}\{T_{ij}=0\}}{\sum_{p'\in\mathcal{P}'}\prod_{i\in S_{p'}}\mathbb{I}\{T_{ij}=0\}+\sum_{p'\in\mathcal{P}'}\prod_{i\in S_{p'}}\mathbb{I}\{T_{ij}=1\}}\right)^2,
\end{eqnarray*}
where $T_{ij}$ is a Bernoulli random variable with parameter $p_i$ under $E_p$. The above formula is lower bounded by
\begin{equation}
\frac{1}{4}\big(6e\big)^{-2n\bar{\nu}-2}P_p\left\{\sum_{p'\in\mathcal{P}'}\prod_{i\in S_{p'}}\mathbb{I}\{T_{ij}=0\}\geq \big(6e\big)^{-n\bar{\nu}-1}\sum_{p'\in\mathcal{P}'}\prod_{i\in S_{p'}}\mathbb{I}\{T_{ij}=1\}\right\}. \label{eq:Bayeslower}
\end{equation}
Hence, it suffices to lower bound the probability of the event above. Suppose $\sum_{i\in[n]}\mathbb{I}\{T_{ij}=0\}\geq \frac{1}{2}\Big(n-\ceil{n\bar{\nu}}\Big)$. Namely, there are at least $\frac{1}{2}\Big(n-\ceil{n\bar{\nu}}\Big)$ zeros  and at most $\frac{1}{2}\Big(n+\ceil{n\bar{\nu}}\Big)$ ones in the sequence $\{T_{ij}\}_{i\in[n]}$. Because $|S_{p'}|=\ceil{n\bar{\nu}}$ for each $p'\in\mathcal{P}'$, we have
$$\sum_{p'\in\mathcal{P}'}\prod_{i\in S_{p'}}\mathbb{I}\{T_{ij}=0\}\geq {\frac{1}{2}\big(n-\ceil{n\bar{\nu}}\big)\choose \ceil{n\bar{\nu}}}\geq \left(\frac{\frac{1}{2}\big(n-\ceil{n\bar{\nu}}\big)}{\ceil{n\bar{\nu}}}\right)^{\ceil{n\bar{\nu}}},$$
$$\sum_{p'\in\mathcal{P}'}\prod_{i\in S_{p'}}\mathbb{I}\{T_{ij}=1\}\leq {\frac{1}{2}\big(n+\ceil{n\bar{\nu}}\big)\choose \ceil{n\bar{\nu}}}\leq \left(\frac{\frac{1}{2}e\big(n+\ceil{n\bar{\nu}}\big)}{\ceil{n\bar{\nu}}}\right)^{\ceil{n\bar{\nu}}}.$$
This implies
$$\frac{\sum_{p'\in\mathcal{P}'}\prod_{i\in S_{p'}}\mathbb{I}\{T_{ij}=0\}}{\sum_{p'\in\mathcal{P}'}\prod_{i\in S_{p'}}\mathbb{I}\{T_{ij}=1\}}\geq \left(\frac{n-\ceil{n\bar{\nu}}}{e(n+\ceil{n\bar{\nu}})}\right)^{\ceil{n\bar{\nu}}}\geq \left(\frac{n-n\bar{\nu}-1}{e(n+n\bar{\nu})}\right)^{n\bar{\nu}+1},$$
which is greater than $\big(6e\big)^{-n\bar{\nu}-1}$ under the assumption that $\bar{\nu}<1/2$ and $n\geq 4$. The above argument implies that
$$\left\{\frac{\sum_{p'\in\mathcal{P}'}\prod_{i\in S_{p'}}\mathbb{I}\{T_{ij}=0\}}{\sum_{p'\in\mathcal{P}'}\prod_{i\in S_{p'}}\mathbb{I}\{T_{ij}=1\}}\geq (6e)^{-n\bar{\nu}-1}\right\}\supset \left\{\sum_{i\in[n]}\mathbb{I}\{T_{ij}=0\}\geq \frac{1}{2}\Big(n-\ceil{n\bar{\nu}}\Big)\right\}.$$
Thus, (\ref{eq:Bayeslower}) is lower bounded by
\begin{eqnarray*}
&& \frac{1}{4}\big(6e\big)^{-2n\bar{\nu}-2}P_p\left\{\sum_{i\in[n]}\mathbb{I}\{T_{ij}=0\}\geq \frac{1}{2}\Big(n-\ceil{n\bar{\nu}}\Big)\right\} \\
&\geq& \frac{1}{4}\big(6e\big)^{-2n\bar{\nu}-2}P_p\left\{\sum_{i\in S_p^c}\mathbb{I}\{T_{ij}=0\}\geq \frac{1}{2}\Big(n-\ceil{n\bar{\nu}}\Big)\right\} \geq \frac{1}{8}\big(6e\big)^{-2n\bar{\nu}-2}.
\end{eqnarray*}
The last inequality is because under $P_p$, The random variable $\sum_{i\in S_p^c}\mathbb{I}\{T_{ij}=0\}$ is  Binomial distribution with mean $\frac{1}{2}\Big(n-\ceil{n\bar{\nu}}\Big)$. Thus, the probability that it is no less than its mean is no less than $1/2$. To summarize, for each $p\in\mathcal{P}'$ and $j\in[m]$, we have
$$\frac{1}{2}E_{p,y_j=1}(\hat{y}_j-1)^2+\frac{1}{2}E_{p,y_j=0}(\hat{y}_j-0)^2\geq \frac{1}{8}\big(6e\big)^{-2n\bar{\nu}-2}\geq\frac{1}{8(6e)^2}\exp\big(-6n\bar{\nu}\big).$$
Taking average over $p\in\mathcal{P}'$ and $j\in[m]$ gives the lower bound for (\ref{eq:bayesriskinproof}). This implies the desired lower bound for the minimax risk.
\end{proof}

\begin{proof}[Proof of Theorem \ref{thm:lower2}]
Remembering the parameter space $\mathcal{P}_{\bar{\mu}}$ defined in (\ref{eq:paraspacemu}) and the least favorable case $p'$ defined in (\ref{eq:lfmu}). It is easy to see that $p'\in\mathcal{P}_{\bar{\mu}}$. Therefore, we have the reduction
$$\sup_{y\in\{0,1\}^m,p\in\mathcal{P}_{\bar{\mu}}}E_{p,y}\left(\frac{1}{m}\sum_{j\in[m]}|\hat{y}_j-y_j|\right)\geq \sup_{y\in\{0,1\}^m}E_{p',y}\left(\frac{1}{m}\sum_{j\in[m]}|\hat{y}_j-y_j|\right),$$
for any $\hat{y}\in[0,1]^m$.
Similar to what we have done in the proof of Theorem \ref{thm:lower1}, by using a uniform prior on $\{0,1\}^m$, the maximum risk is lower bounded by the Bayes risk
\begin{eqnarray*}
 \sup_{y\in\{0,1\}^m}E_{p',y}\left(\frac{1}{m}\sum_{j\in[m]}|\hat{y}_j-y_j|\right) \geq \frac{1}{m}\sum_{j\in[m]}\Bigg(\frac{1}{2}E_{p',y_j=1}(\hat{y}_j-1)^2+\frac{1}{2}E_{p',y_j=0}(\hat{y}_j-0)^2\Bigg),
\end{eqnarray*}
where the Bayes risk on the right hand side of the above inequality is minimized by the Bayes solution
\begin{eqnarray*}
\hat{y}_j &\propto& \prod_{i\in[n]}\bar{\mu}^{X_{ij}}(1-\bar{\mu})^{1-X_{ij}}, \\
1-\hat{y}_j &\propto& \prod_{i\in[n]}\bar{\mu}^{1-X_{ij}}(1-\bar{\mu})^{X_{ij}}.
\end{eqnarray*}
By symmetry and (\ref{eq:representation}), for each $j\in[m]$, we have
\begin{eqnarray}
\nonumber && \frac{1}{2}E_{p',y_j=1}(\hat{y}_j-1)^2+\frac{1}{2}E_{p',y_j=0}(\hat{y}_j-0)^2 \\
\label{eq:muBayesrisk} &=& E_{p'}\left(1+\exp\Bigg(\sum_i (2T_{ij}-1)\log\frac{\bar{\mu}}{1-\bar{\mu}}\Bigg)\right)^{-2}.
\end{eqnarray}
We wish we could move the expectation onto the exponent. Define $g(x)=(1+e^x)^{-2}$. Then we have
$$g''(x)=\frac{2e^x(2e^x-1)}{(1+e^x)^4}\geq 0,\quad\text{when }x\geq -\log 2.$$
Certainly $g(x)$ is a convex function when $x\geq 0$. Define the event $A_j=\left\{\sum_i(2T_{ij}-1)\geq 0\right\}$, and then (\ref{eq:muBayesrisk}) can be lower bounded by
\begin{eqnarray}
\nonumber && E_{p'}\left[g\left(\sum_i (2T_{ij}-1)\log\frac{\bar{\mu}}{1-\bar{\mu}}\right)\Big| A_j\right] P_{p'}(A_j) \\
\label{eq:lowermuBayes} &\geq& g\left(\log\frac{\bar{\mu}}{1-\bar{\mu}}E_{p'}\left[\sum_i (2T_{ij}-1)\Big|A_j\right]\right)P_{p'}(A_j),
\end{eqnarray}
where we have used Jensen's inequality for conditional expectation. From (\ref{eq:lowermuBayes}), it suffices to lower bound $P_{p'}(A_j)$ and upper bound $E_{p'}\left[\sum_i (2T_{ij}-1)\Big|A_j\right]$. Using Lemma \ref{lem:hoeffding}, we have
$$P_{p'}(A_j)\geq 1-\exp\Big(-2n\big(\bar{\mu}-1/2\big)^2\Big)\geq 1-e^{-n/8},$$
where we have used the assumption $\bar{\mu}\geq 3/4$. By the fact that
$$E_{p'}\left[\sum_i (2T_{ij}-1)\right]=E_{p'}\left[\sum_i (2T_{ij}-1)\Big|A_j\right]P_{p'}(A_j)+E_{p'}\left[\sum_i (2T_{ij}-1)\Big|A_j^c\right]P_{p'}(A_j^c),$$
we have
\begin{eqnarray*}
&& E_{p'}\left[\sum_i (2T_{ij}-1)\Big|A_j\right] \\
&=& \frac{E_{p'}\left[\sum_i (2T_{ij}-1)\right]-E_{p'}\left[\sum_i (2T_{ij}-1)\Big|A_j^c\right]P_{p'}(A_j^c)}{P_{p'}(A_j)} \\
&\leq& \frac{n(2\bar{\mu}-1)+ne^{-n/8}}{1-e^{-n/8}} \leq \frac{n(2\bar{\mu}-1)+n(2\bar{\mu}-1)}{1/2}=4n(2\bar{\mu}-1),
\end{eqnarray*}
by $e^{-n/8}\leq 1/2\leq (2\bar{\mu}-1)$ under the assumption that $n\geq 6$ and $\bar{\mu}\geq 3/4$. Using the lower bound of $P_{p'}(A_j)$ and the upper bound of $E_{p'}\left[\sum_i (2T_{ij}-1)\Big|A_j\right]$, we can lower bound (\ref{eq:lowermuBayes}) by
\begin{eqnarray*}
&& \frac{1}{2}g\left(4n(2\bar{\mu}-1)\log\frac{\bar{\mu}}{1-\bar{\mu}}\right) = \frac{1}{2}g\Big(4nD(\bar{\mu}||1-\bar{\mu})\Big) \\
&=& \frac{1}{2}\left(1+\exp\Big(4nD(\bar{\mu}||1-\bar{\mu})\Big)\right)^{-2} \geq \frac{1}{8}\exp\Big(-8nD(\bar{\mu}||1-\bar{\mu})\Big),
\end{eqnarray*}
where the last inequality is
because $1\leq \exp\Big(4nD(\bar{\mu}||1-\bar{\mu})\Big)$ for all $\bar{\mu}\geq 3/4$ and $n\geq 6$. Averaging over $j\in[m]$, $\frac{1}{8}\exp\Big(-8nD(\bar{\mu}||1-\bar{\mu})\Big)$ is a lower bound for the minimax risk, and the proof is complete.
\end{proof}

\section{Proofs of Theorem \ref{thm:votebad}, Theorem \ref{thm:MLEbad} and Theorem \ref{thm:votinggood}}

In this section, we gather the proofs for the results in Section \ref{sec:compare}.  To prove Theorem \ref{thm:votebad}, we need Berry-Esseen bound for the normal approximation. The best constant obtained so far for the Berry-Esseen bound is given by \cite{shevtsova11}, and the result is presented in the following lemma.

\begin{lemma}[\cite{shevtsova11}] \label{lem:berry}
Let $U_1, U_2,...,U_n$ be i.i.d. random variables with mean $0$ and variance $\sigma^2$. Define $F_n(t)=\mathbb{P}\Bigg(\frac{1}{\sigma\sqrt{n}}\sum_i U_i\leq t\Bigg)$. Then, we have
$$\sup_{t\in\mathbb{R}}\left|F_n(t)-\Phi(t)\right|\leq\frac{c\mathbb{E}|U_1|^3}{\sigma\sqrt{n}},$$
where $c<0.4748$ and $\Phi(t)$ is the cumulative distribution function of $N(0,1)$.
\end{lemma}

\begin{proof}[Proof of Theorem \ref{thm:votebad}]
Let $\hat{y}$ be the majority voting estimator.
By the definition of majority voting and the representation (\ref{eq:representation}), we have $|\hat{y}_j-y_j^*|=\mathbb{I}\left\{\frac{1}{n}\sum_i T_{ij}< \frac{1}{2}\right\}$. Define $\{T_i\}_{i\in[n]}$ to be independent Bernoulli random variable with parameter $p_i^*$, and we have
\begin{equation}
\frac{1}{m}\sum_j \mathbb{E}|\hat{y}_j-y_j^*|=\frac{1}{m}\sum_j\mathbb{P}\left\{\frac{1}{n}\sum_i T_{ij}< \frac{1}{2}\right\}=\mathbb{P}\left\{\frac{1}{n}\sum_i T_{i}< \frac{1}{2}\right\}. \label{eq:MViid}
\end{equation}
Without loss of generality, we let $p_i^*=1/2$ for $i\leq n-\ceil{n^{\delta}}$ and $p_i^*=1$ for $i>n-\ceil{n^{\delta}}$. Therefore,
\begin{equation}
\mathbb{P}\left\{\frac{1}{n}\sum_i T_{i}< \frac{1}{2}\right\}=\mathbb{P}\left\{\frac{2}{\sqrt{n-\ceil{n^{\delta}}}}\sum_{i\leq n-\ceil{n^{\delta}}}\Bigg(T_i-\frac{1}{2}\Bigg)\leq -\frac{\ceil{n^{\delta}}}{\sqrt{n-\ceil{n^{\delta}}}}\right\}. \label{eq:MV1}
\end{equation}
Using Lemma \ref{lem:berry}, we have
\begin{equation}
\sup_t\left|\mathbb{P}\left\{\frac{2}{\sqrt{n-\ceil{n^{\delta}}}}\sum_{i\leq n-\ceil{n^{\delta}}}\Bigg(T_i-\frac{1}{2}\Bigg)\leq t\right\}-\Phi(t)\right|\leq \frac{1}{16}\big(n-\ceil{n^{\delta}}\big)^{-1/2}. \label{eq:MV2}
\end{equation}
Combining (\ref{eq:MV1}) and (\ref{eq:MV2}), when $\delta\in (0,1/2)$, we have
$$\mathbb{P}\left\{\frac{1}{n}\sum_i T_{i}< \frac{1}{2}\right\}= \lim_{n\rightarrow \infty}\Phi\left(-\frac{\ceil{n^{\delta}}}{\sqrt{n-\ceil{n^{\delta}}}}\right)+o(1)=\frac{1}{2}-o(1).$$
When $\delta\in (1/2,1)$, we have
$$\mathbb{P}\left\{\frac{1}{n}\sum_i T_{i}< \frac{1}{2}\right\}= \lim_{n\rightarrow\infty}\Phi\left(-\frac{\ceil{n^{\delta}}}{\sqrt{n-\ceil{n^{\delta}}}}\right)+o(1)=o(1).$$
When $\delta=1/2$, we have
$$\mathbb{P}\left\{\frac{1}{n}\sum_i T_{i}< \frac{1}{2}\right\}=\lim_{n\rightarrow\infty}\Phi\left(-\frac{\ceil{n^{\delta}}}{\sqrt{n-\ceil{n^{\delta}}}}\right)+o(1)=\Phi(-1)+o(1).$$
Thus, the proof is complete.
\end{proof}

\begin{proof}[Proof of Theorem \ref{thm:MLEbad}]
For $T_{ij}$ in (\ref{eq:repres}), define $p_{ij}^*=\mathbb{E}T_{ij}$.
For simplicity of notation, we use $\hat{y},\hat{p},\hat{r}$ to denote $y^{(t)}, \check{p}, r^{(t)}$.  By (\ref{eq:handT}) and (\ref{eq:sumhandT}), we have
$\left|\hat{p}_i-\frac{1}{m}\sum_jp_{ij}^*\right|\leq \hat{r}+\left|\frac{1}{m}\sum_j(T_{ij}-p_{ij}^*)\right|$. Lemma \ref{lem:hoeffding} together with union bound gives $\max_{i\in[n]}\left|\frac{1}{m}\sum_j(T_{ij}-p_{ij}^*)\right|\leq\sqrt{\frac{\log m}{m}}$ with probability at least $1-C/m$. When $\hat{r}\geq m^{-1/2}$, we reach the conlusion. Thus, let us from now on consider the case $\hat{r}\leq m^{-1/2}$. Therefore, we have
\begin{equation}
\max_{i\in[n]}\left|\hat{p}_i-\frac{1}{m}\sum_jp_{ij}^*\right|\leq 2\sqrt{\frac{\log m}{m}}. \label{eq:MLE1}
\end{equation}
Under the current setting, direct calculation gives
\begin{equation}
\max_{i\in G_1}\left|\frac{1}{m}\sum_jp_{ij}^*-\frac{4}{5}\right|\leq 2m^{-1/2}\quad\text{and}\quad\max_{i\in G_2}\left|\frac{1}{m}\sum_jp_{ij}^*-\frac{1}{2}\right|\leq 2m^{-1/2}. \label{eq:MLE2}
\end{equation}
Combining (\ref{eq:MLE1}) and (\ref{eq:MLE2}), we have $\max_{i\in G_1}|\hat{p}_i-4/5|\leq 4\sqrt{\frac{\log m}{m}}$ and $\max_{i\in G_2}|\hat{p}_i-1/2|\leq 4\sqrt{\frac{\log m}{m}}$. Furthermore, by Proposition \ref{prop:log}, we have
\begin{equation}
\max_{i\in G_1}\left|\log\frac{\hat{p}_i}{1-\hat{p}_i}-\log 4\right|\leq 80\sqrt{\frac{\log m}{m}}\quad\text{and}\quad \max_{i\in G_2}\left|\log\frac{\hat{p}_i}{1-\hat{p}_i}\right|\leq 80\sqrt{\frac{\log m}{m}}, \label{eq:MLE3}
\end{equation}
with probability at least $1-C/m$
once $\hat{r}\leq m^{-1/2}$ holds. Define the event
$$E_j=\left\{\log 4\sum_{i\in G_1}(2T_{ij}-1)\leq -80n\sqrt{\frac{\log m}{m}}\right\},\quad\text{for }j\in[m].$$
As long as $E_j$ and (\ref{eq:MLE3}) holds, we have
\begin{eqnarray*}
&& \sum_i(2T_{ij}-1)\log\frac{\hat{p}_i}{1-\hat{p}_i} =\left( \sum_{i\in G_1}+\sum_{i\in G_2}\right)(2T_{ij}-1)\log\frac{\hat{p}_i}{1-\hat{p}_i}  \\
&\leq& \log 4\sum_{i\in G_1}(2T_{ij}-1)+ 80n\sqrt{\frac{\log m}{m}}\leq 0.
\end{eqnarray*}
This implies
$$|\hat{y}_j-y_j^*|=\frac{1}{1+\exp\Big( \sum_i(2T_{ij}-1)\log\frac{\hat{p}_i}{1-\hat{p}_i}\Big)}\geq\frac{1}{2},$$
so that the $j$-th item is mis-labeled. Now we are going to control the count $\sum_{j\in S_2}\mathbb{I}_{E_j}$. Using Lemma \ref{lem:berry}, we have $\mathbb{P}(E_j)\geq 0.49$ under the current setting for sufficiently large $n$ and $m$. Then, by Lemma \ref{lem:hoeffding}, we have $\mathbb{P}\Big(\frac{1}{m_2}\sum_{j\in S_2}\mathbb{I}_{E_j}<1/4\Big)\leq 0.68$. That is to say, $\sum_{j\in S_2}\mathbb{I}_{E_j}\geq \frac{1}{4}m_2$ with probability at least $0.32$. The fact that $m_2/4$ items are mislabeled implies $\frac{1}{m}\sum_j|\hat{y}_j-y_j^*|\geq \frac{1}{8}m^{-1/2}$. To summarize, once $\hat{r}\leq m^{-1/2}$ holds, we must have $\hat{r}\geq \frac{1}{8}m^{-1/2}$ with probability at least $0.32-C/m\geq 0.3$ for sufficiently large $m$. Thus, the proof is complete.
\end{proof}

\begin{proof}[Proof of Theorem \ref{thm:votinggood}]
Without loss of generality, let $n/2$ be an integer. By (\ref{eq:repres}), we have
$\frac{1}{m}\sum_j|\hat{y}_j-y_j^*|=\frac{1}{m}\sum_j\eta_j$,
where $\eta_j=\mathbb{I}\left\{\frac{1}{n}\sum_iT_{ij}<1/2\right\}$.
By Lemma \ref{lem:hoeffding}, $\mathbb{E}\eta_j\leq \exp\Big(-\frac{9}{200}n\Big)$ for all $j\in[m]$. By Markov's inequality, we have
$$\mathbb{P}\left(\frac{1}{m}\sum_j|\hat{y}_j-y_j^*|>\exp\Big(-\frac{1}{25}n\Big)\right)\leq \exp\Big(\frac{1}{25}n\Big)\frac{1}{m}\sum_j\mathbb{E}\eta_j\leq \exp\Big(-\frac{1}{200}n\Big).$$
The proof is complete.
\end{proof}

\section*{Acknowledgements}

The authors thank Nihar Shah for insightful discussion. The authors are grateful for the suggestions made by two anonymous referees, which lead to significant improvement of the paper. The authors thank Derek Feng for their efforts in helping with the English.

\bibliography{ref}


\end{document}